\pdfoutput=1
\documentclass[11pt,a4paper]{article}

\usepackage[margin=1in]{geometry}
\usepackage{amsmath,amssymb,amsfonts,amsthm,mathtools}
\usepackage{bm}
\usepackage{booktabs}
\usepackage{array}
\usepackage{enumitem}
\PassOptionsToPackage{hyphens}{url}
\usepackage{hyperref}
\usepackage{xcolor}
\usepackage{graphicx}
\usepackage{tikz}
\usetikzlibrary{positioning,arrows.meta}
\usepackage[edges]{forest}
\usepackage{longtable}
\usepackage{algorithm}
\usepackage{algpseudocode}
\usepackage{listings}

\hypersetup{colorlinks=true, linkcolor=blue!50!black, citecolor=blue!50!black, urlcolor=blue!50!black,
  pdftitle={PAA: The Probabilistic Allen Algebra -- A generative and complete probabilistic extension of Allen's interval relations},
  pdfauthor={Julian Eggert},
  pdfkeywords={Allen interval algebra, probabilistic temporal reasoning, Gaussian orthant probabilities, temporal uncertainty, vague temporal language}}

\title{PAA: The Probabilistic Allen Algebra\\[0.4em]
  \large A generative and complete probabilistic extension of Allen's interval relations}
\author{Julian Eggert\thanks{Honda Research Institute, Carl-Legien-Str.\ 30, 63073 Offenbach, Germany. Email: \href{mailto:julian.eggert@honda-ri.de}{julian.eggert@honda-ri.de}. Accompanying open-source implementation: \url{https://github.com/HRI-EU/probabilistic-allen-algebra}}}
\date{September 17, 2026}

\providecommand{\Description}[2][]{}

\newtheorem{proposition}{Proposition}

\newcommand{\N}{\mathcal{N}}
\newcommand{\Prob}{\mathbb{P}}
\newcommand{\E}{\mathbb{E}}
\newcommand{\PhiCDF}{\Phi}

\newcommand{\given}{\,|\,}
\newcommand{\dd}{\,\mathrm{d}}

\newcommand{\transpose}{^{\mathsf T}}
\newcommand{\erfop}{\operatorname{erf}}
\newcommand{\erfcop}{\operatorname{erfc}}
\newcommand{\before}{\texttt{before}}
\newcommand{\after}{\texttt{after}}
\newcommand{\meets}{\texttt{meets}}
\newcommand{\metby}{\texttt{met\_by}}
\newcommand{\overlaps}{\texttt{overlaps}}
\newcommand{\overlappedby}{\texttt{overlapped\_by}}
\newcommand{\during}{\texttt{during}}
\newcommand{\contains}{\texttt{contains}}
\newcommand{\starts}{\texttt{starts}}
\newcommand{\startedby}{\texttt{started\_by}}
\newcommand{\finishes}{\texttt{finishes}}
\newcommand{\finishedby}{\texttt{finished\_by}}
\newcommand{\equals}{\texttt{equals}}
\newcommand{\argmax}{\operatorname*{arg\,max}}

\newcommand{\ipair}[4]{%
  \begin{tikzpicture}[baseline=-0.6ex, x=1.5mm, y=1.5mm]
    \draw[fill=black!22] (#1,1.2) rectangle (#2,2.1);
    \draw[fill=white] (#3,0.0) rectangle (#4,0.9);
  \end{tikzpicture}%
}
\newcommand{\ipairf}[2]{%
  \begin{tikzpicture}[baseline=-0.6ex, x=1.5mm, y=1.5mm]
    \draw[fill=black!22] (#1,1.2) rectangle (#2,2.1);
    \draw[fill=white] (2,0.0) rectangle (4,0.9);
    \draw[black!60, densely dotted, line width=0.4pt] (2,-0.45)--(2,2.55);
    \draw[black!60, densely dotted, line width=0.4pt] (4,-0.45)--(4,2.55);
  \end{tikzpicture}%
}
\newcommand{\taxcellf}[3]{\begin{tabular}{@{}c@{}}\ipairf{#1}{#2}\\[1pt]{\scriptsize #3}\end{tabular}}
\newcommand{\taxfam}[3]{\fcolorbox{black!35}{#1}{\begin{tabular}{@{}c@{}}{\scriptsize #2}\\[3pt]#3\end{tabular}}}

\begin{document}
\maketitle

\begin{abstract}
Allen's interval algebra provides a qualitative calculus for temporal relations, but its base relations are crisp predicates over exact interval boundaries. This is inadequate for temporal information extracted from language, perception, databases, or uncertain histories, where event times, durations, and boundaries are uncertain and where expressions such as \emph{just before}, \emph{roughly during}, or \emph{starting with} have graded semantics. We develop \emph{the probabilistic Allen algebra}: a generative and complete extension of Allen's calculus in which relation probabilities are \emph{derived} from distributions over interval boundaries rather than assigned as scores. Time points are Gaussian variables; intervals are represented by Gaussian midpoints and non-negative durations obtained by truncating latent Gaussian duration variables. Each relation is then a measurable boundary-ordering predicate in one common temporal probability space, rather than an independently assigned class score. Point--point relations admit closed forms using the error function and complementary error function. Point--interval and interval--interval relations become multivariate Gaussian orthant probabilities induced by linear inequalities. Equality and contact relations, such as \meets, \starts, \finishes, and \equals, are treated through confidence borders or soft membership kernels. The construction \emph{derives} Allen's taxonomy rather than positing it: the thirteen relations arise as the sign-partition cells of the boundary geometry, their coarse predicates such as precedence, overlap, and containment are unions of leaves whose probabilities are the corresponding leaf sums, and this hierarchy is the invariant preserved as intervals collapse to points and the thirteen relations reduce to five and then three. For temporal language, we distinguish contact or boundary-sharing relations from verbal grading: expressions such as \emph{shortly before}, \emph{rather before}, or \emph{long before} can be modeled by the graded strength of a relation family relative to alternatives, while \meets, \starts, and \finishes\ remain contact relations controlled by tolerance parameters. Because the Gaussian relation probabilities depend only on standardized temporal separations, jointly scaling means, durations, uncertainty, and tolerances leaves every probability unchanged: the algebra is scale-invariant. Under a single tolerance the thirteen relations form a true partition that recovers crisp Allen as the tolerance vanishes. As in well-established logical reasoning frameworks such as CIDOC~CRM, where relations are expressed through primitives, each relation decomposes into elementary boundary comparisons (temporal primitives), here lifted to a correlation-aware probabilistic form. The model therefore provides a robust bridge between metric temporal uncertainty, qualitative interval reasoning, and vague temporal adverbials in natural language. Beyond the complete theory, we provide an open, fully tested implementation that computes the relation probabilities directly.

\end{abstract}

\tableofcontents

\section{Introduction}

\subsection{Motivation}
\label{sec:motivation}

Temporal reasoning has two complementary traditions. In symbolic artificial intelligence and knowledge representation, qualitative temporal calculi abstract away metric time and reason with relations between intervals. Allen's interval algebra is the canonical example: it represents the configuration of two intervals using relations such as \before, \meets, \overlaps, \during, \starts, and \finishes\ \cite{allen1983maintaining}. In probabilistic temporal reasoning, by contrast, probability is often used to describe how states or events evolve over time, for example in temporal graphical models, stochastic processes, filtering, and probabilistic temporal logics \cite{hanks1994probabilistic}.

A different and less uniformly developed problem arises when the temporal objects themselves are uncertain. An event may be known to have happened, but its occurrence time may be imprecise; an interval may have uncertain start, end, midpoint, or duration; and a natural-language expression may communicate a qualitative temporal relation without committing to exact boundaries. Temporal database work has studied valid-time indeterminacy and probabilistic temporal data \cite{dyreson1998valid,dekhtyar2001probabilistic}, and probabilistic variants of interval algebra have been proposed \cite{mouhoub2008probabilistic,petridis2010hourglass}. However, much of this work either assigns probabilities to relations directly or treats relation prediction as a flat classification problem over the thirteen Allen labels.

Two failure modes of crisp relations motivate a graded treatment. First, relations are \emph{brittle}: a sub-threshold shift of a single boundary flips the qualitative label (Figure~\ref{fig:robustness}). Second, crisp relations are \emph{inadequate}: one label can cover markedly different configurations (Figure~\ref{fig:inadequacy}). In both cases the distribution-induced probabilities developed here vary smoothly and distinguish the situations, as the annotations in the two figures show.

\begin{figure}[t]
\centering
\includegraphics[width=\linewidth]{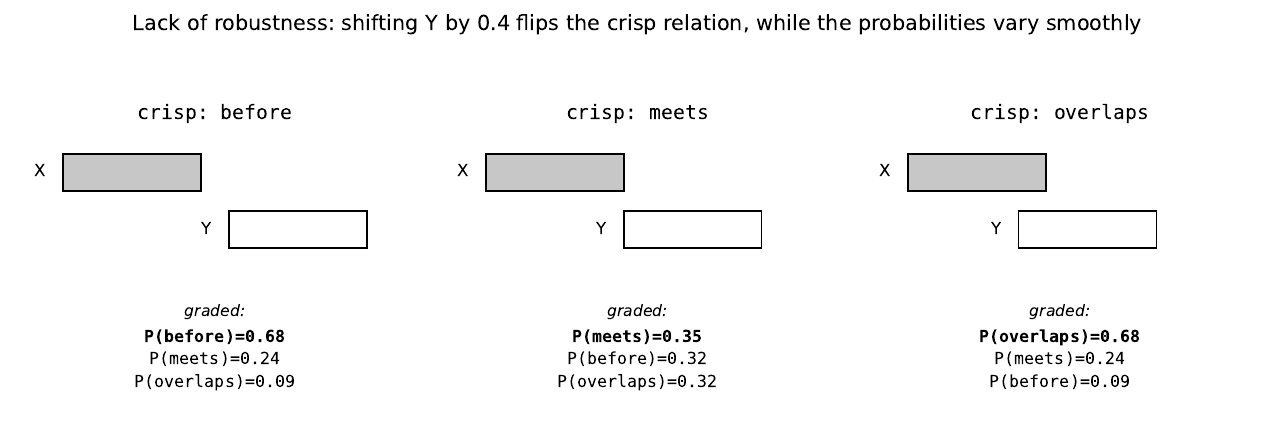}
\caption{Lack of robustness. Shifting \(Y\) by \(0.4\) flips the crisp relation \before\(\to\)\meets\(\to\)\overlaps, while the distribution-induced probabilities (computed with boundary uncertainty) vary smoothly.}
\Description{Two rows of interval bars. Top: three crisp arrangements of X and Y that differ only by a shift of Y, labelled before, meets, and overlaps. Bottom: the corresponding relation probabilities, which change smoothly across the same shift.}
\label{fig:robustness}
\end{figure}

\begin{figure}[t]
\centering
\includegraphics[width=\linewidth]{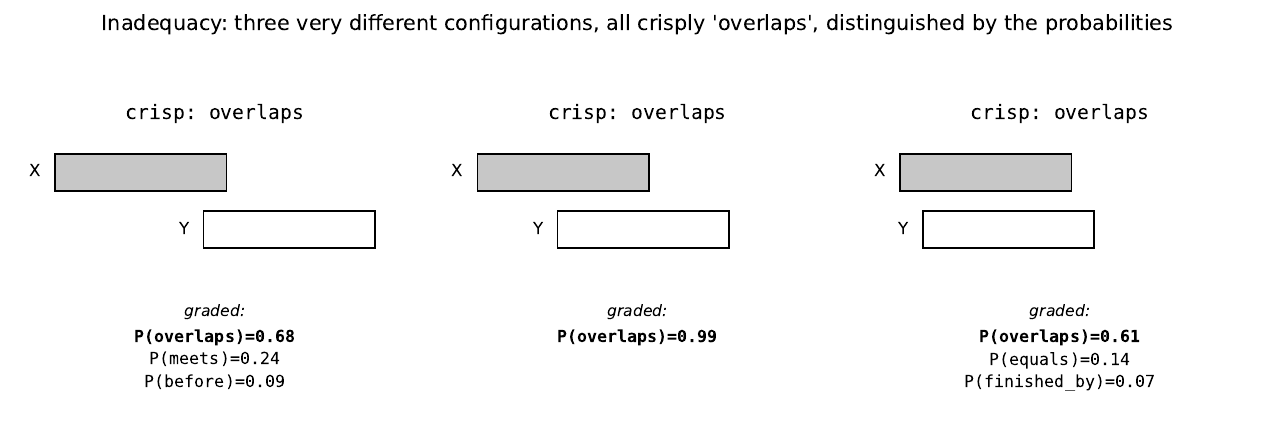}
\caption{Inadequacy. Three configurations are all crisply \overlaps, yet they differ markedly; the probabilities separate the near-\before\ case from the near-\equals\ case.}
\Description{Three interval arrangements that all carry the crisp label overlaps, with bar charts of their relation probabilities showing that one is nearly before and one nearly equals.}
\label{fig:inadequacy}
\end{figure}

This paper argues for a distribution-induced view. The primitive uncertainty lives on the quantities that specify the temporal objects: the location of a time point and, for an interval, any consistent choice among its start, end, midpoint, and duration---these parameterizations are interchangeable, and we adopt midpoint and duration (Section~\ref{sec:parameterization}; Appendix~\ref{app:parameterizations}). Because every Allen relation between two uncertain temporal objects \(X\) and \(Y\) (time points or intervals) is fixed by how their boundary points---both starts and ends---are ordered relative to one another (does \(X\) finish before \(Y\) starts? do they start together?), and each boundary is a linear function of these uncertain quantities, a relation holds on exactly a region of those quantities: the values that produce its ordering. Its probability is then how much of the distribution falls in that region. This has two consequences. First, it gives a direct probabilistic semantics to every relation \(R\):
\begin{equation}
\label{eq:generative}
    P(X\,R\,Y)=\int \mathbf 1[X\,R\,Y]\,p(X,Y)\dd X\dd Y.
\end{equation}
Equation~\eqref{eq:generative} is the conceptual core of the approach. Picture the two uncertain intervals as a recipe for producing concrete pairs: each time we draw from the uncertainty we obtain two ordinary intervals with definite start and end points. For any such definite pair a plain Allen relation either holds or it does not---either \(X\) lies entirely before \(Y\), say, or it does not. The symbol \(\mathbf 1[X\,R\,Y]\) simply records this yes/no answer: it equals \(1\) when the relation holds and \(0\) when it does not. The density \(p(X,Y)\) records how plausible each concrete pair is. The integral multiplies these two quantities and sums over all possible pairs, thereby accumulating the plausibility of exactly those pairs for which the answer is ``yes.'' The result is the probability that the relation holds.

Equivalently, and more concretely: if we generated a large collection of concrete interval pairs according to the assumed uncertainty and checked the ordinary Allen relation for each, then \(P(X\,R\,Y)\) would be the fraction of the collection in which \(R\) is true. Because every relation is a yes/no question about the same underlying pair, the thirteen probabilities are shares of one whole and add up to one; and none of them is assigned by hand---each is fixed entirely by the uncertainty placed on the time points and durations. The remainder of the paper turns this definition into closed-form expressions by taking the time points and durations to be Gaussian, under which the ``yes'' region is described by a few linear inequalities and the integral reduces to a standard multivariate Gaussian probability.

This strict probabilistic approach makes the taxonomy of temporal relations explicit. The Allen leaves are not independent classes. Coarser predicates such as precedence, non-disjointness, containment, and contact are unions of leaves, and their probabilities are obtained by summing the probabilities of their refinements. Thus, if \meets\ holds, then the coarser predicate ``precedes or touches'' also holds; if containment is probable, one can refine it into \starts, \during, \finishes, or \equals. This taxonomic view is essential for robustness. A small perturbation of boundaries can transform a crisp \before\ relation into \meets, or a \during\ relation into \starts. A flat classifier sees this as a label discontinuity. A hierarchical probabilistic model sees it as a transfer of probability mass within a stable parent predicate.

For NLP, the same hierarchy suggests a separation between qualitative relation families and verbal grading. Contact relations such as \meets\ or \starts\ encode boundary correspondence and therefore require a margin or tolerance in continuous probability spaces. By contrast, many adverbials \emph{grade} a temporal relation rather than asserting contact. Expressions such as \emph{just before}, \emph{shortly before}, \emph{rather before}, \emph{some time before}, and \emph{long before} all describe how one event stands in time relative to another---they are graded forms of the \before\ family of Allen's relations---and differ only in degree, not in whether the boundaries touch. Spatial language has the same structure: \emph{just above}, \emph{far to the left of}, and \emph{right beside} grade qualitative spatial relations in the same way. Such adverbials express the degree to which a coarse relation family, for example precedence, is supported relative to its alternatives---that is, the relation probabilities themselves. Since these probabilities are dimensionless and depend on metric distances only after normalization by temporal uncertainty and duration scale, the model naturally normalizes context-sensitivity: one month before a two-month interval can receive the same graded beforeness as one hour before a two-hour interval when the underlying parameters are scaled consistently.

\subsection{Contributions}

The contributions are:
\begin{enumerate}[leftmargin=*,label=(C\arabic*)]
    \item We formalize uncertain time points as Gaussian variables and uncertain intervals as Gaussian midpoints with non-negative truncated-Gaussian durations.
    \item We derive closed point--point probabilities using \(\operatorname{erf}\) and \(\operatorname{erfc}\), including confidence borders for equality.
    \item We derive point--interval relations as bivariate Gaussian CDFs in the general truncated-duration case, with error-function limits for deterministic durations.
    \item We derive interval--interval Allen relation probabilities as multivariate Gaussian CDFs over linear boundary inequalities.
    \item We introduce confidence borders for equality/contact relations and identify outer adjacency (\meets, \metby) and inner adjacency (\starts, \startedby, \finishes, \finishedby, \equals).
    \item We show Allen's relation taxonomy is \emph{derived, not stipulated}: the thirteen leaves are the sign-partition cells of the boundary-difference geometry, their coarse predicates are unions of leaves whose probabilities are the leaf sums, and this hierarchy is the invariant preserved as intervals collapse to points.
    \item We position the approach relative to Allen's Hourglass, arguing that our model gives distributional semantics and a principled treatment of equality/contact while preserving the robustness motivation.
    \item We discuss the bidirectional mapping between metric temporal uncertainty, graded relation membership, and vague temporal adverbials in NLP.
    \item We identify scale invariance as a desirable normalization property: verbal grading can be tied to relation probabilities rather than absolute distances, thereby reducing context-sensitivity across temporal scales.
\end{enumerate}

\section{Previous Work}

\subsection{Allen Interval Algebra}

Allen's interval algebra defines thirteen basic relations between two intervals \(X=[a_X,b_X]\) and \(Y=[a_Y,b_Y]\), assuming \(a_X\le b_X\) and \(a_Y\le b_Y\) \cite{allen1983maintaining}. The relations are mutually exclusive and jointly exhaustive for definite intervals. Table~\ref{tab:allen-relations} gives the standard orientation \(X\,R\,Y\), with the boundary ordering and an example arrangement (\(X\) filled, \(Y\) outlined) for each.

\begin{table}[h]
\centering
\begin{tabular}{l l c}
\toprule
Relation & Boundary ordering & Example \\
\midrule
\before       & \(a_X<b_X<a_Y<b_Y\) & \ipair{0}{2}{3}{5} \\
\meets        & \(a_X<b_X=a_Y<b_Y\) & \ipair{0}{2}{2}{4} \\
\overlaps     & \(a_X<a_Y<b_X<b_Y\) & \ipair{0}{3}{2}{5} \\
\starts       & \(a_X=a_Y<b_X<b_Y\) & \ipair{0}{2}{0}{4} \\
\during       & \(a_Y<a_X<b_X<b_Y\) & \ipair{1}{3}{0}{5} \\
\finishes     & \(a_Y<a_X<b_X=b_Y\) & \ipair{2}{4}{0}{4} \\
\equals       & \(a_X=a_Y<b_X=b_Y\) & \ipair{0}{4}{0}{4} \\
\finishedby   & \(a_X<a_Y<b_X=b_Y\) & \ipair{0}{4}{2}{4} \\
\contains     & \(a_X<a_Y<b_Y<b_X\) & \ipair{0}{5}{1}{3} \\
\startedby    & \(a_X=a_Y<b_Y<b_X\) & \ipair{0}{4}{0}{2} \\
\overlappedby & \(a_Y<a_X<b_Y<b_X\) & \ipair{2}{5}{0}{3} \\
\metby        & \(a_Y<b_Y=a_X<b_X\) & \ipair{2}{4}{0}{2} \\
\after        & \(a_Y<b_Y<a_X<b_X\) & \ipair{3}{5}{0}{2} \\
\bottomrule
\end{tabular}
\caption{Allen's thirteen base relations: boundary ordering and an example arrangement of \(X\) (filled) and \(Y\) (outlined).}
\label{tab:allen-relations}
\end{table}

Allen's calculus is usually described as qualitative: it suppresses metric distances while retaining the ordering of boundaries. Subsequent work has studied composition tables, tractable subclasses, constraint propagation, and encodings for planning and scheduling \cite{janhunen2019allen}; representative results characterize maximal tractable subclasses of Allen's algebra, including extensions that admit metric time \cite{drakengren1997eight}, and carry the same constraint-based reasoning to qualitative spatial calculi \cite{renz2001efficient}. These methods typically assume definite intervals or sets of possible qualitative relations, rather than continuous probability distributions over boundaries: they reason about \emph{which} qualitative relations are consistent, whereas we assign each relation a probability derived from a distribution over the boundaries themselves.

\subsection{Probabilistic Temporal Reasoning and Temporal Indeterminacy}

The phrase probabilistic temporal reasoning is used in at least two senses. The first studies stochastic systems evolving over time: hidden Markov models, dynamic Bayesian networks, Kalman filters, temporal point processes, and stochastic temporal logics. These models ask which state holds, which event occurs, or which temporal property is satisfied.

The second sense concerns uncertainty in temporal objects themselves. Valid-time indeterminacy represents facts whose occurrence time is uncertain \cite{dyreson1998valid}. Probabilistic temporal databases extend this idea toward algebraic operations over temporally uncertain tuples \cite{dekhtyar2001probabilistic}. This second sense is closer to the present work, because the question is not only whether an event occurs, but when it occurred and how its temporal extent relates to another uncertain extent.

\subsection{Probabilistic Extensions of Allen Relations}

Mouhoub and Liu propose a probabilistic interval algebra for managing uncertain temporal relations \cite{mouhoub2008probabilistic}, and probabilistic interval networks attach a probability to each Allen relation and propagate them through a constraint network \cite{ryabov2004probabilistic}. Such work typically treats probabilities as annotations on qualitative relations or constraints. That is useful for temporal constraint networks, but differs from the distribution-induced approach here: we do not assume relation probabilities as primitives; we derive them from distributions over points and durations.

A parallel line softens Allen's relations with fuzzy set theory rather than probability. Badaloni and Giacomin propose a fuzzy extension of Allen's interval algebra in which each of the thirteen base relations carries a graded preference degree, later developed into the constraint algebra \(\mathrm{IA}_{\mathrm{fuz}}\) for qualitative fuzzy temporal reasoning \cite{badaloni2000fuzzy,badaloni2006iafuz}; Schockaert et al. fuzzify the interval relations to compare intervals with imprecise boundaries \cite{schockaert2008fuzzifying}. These approaches share our robustness motivation---a sub-threshold shift of one boundary should not flip a hard label---but the grading is supplied by membership functions chosen on the relations, not induced by a generative model over the underlying metric quantities, and the resulting degrees do not in general form a single probability distribution that sums to one across the thirteen relations. A complementary line in formal semantics models vagueness probabilistically rather than through fixed memberships---Lassiter and Goodman interpret a gradable adjective by Bayesian inference over an uncertain threshold~\cite{lassiter2017adjectival}---but there too the grading attaches to a predicate against a threshold, not to interval relations induced by a distribution over boundary geometry.

Allen's Hourglass is especially relevant \cite{petridis2010hourglass}. It diagnoses the brittleness of crisp Allen relations under noisy intervals and proposes a two-dimensional mapping based on relative position and relative size. The hourglass representation provides a robust probabilistic treatment of relation assignment and makes some symmetries of the Allen relations visually explicit. Its probabilities, however, are defined through primitive position and size relation functions rather than through a generative probability model over boundaries, midpoints, and durations. Moreover, equality-based relations such as contact are intrinsically delicate in continuous probability spaces. Our confidence-border treatment directly addresses this issue by replacing zero-measure boundary equalities with tolerance bands or soft membership kernels.

\subsection{Usage and Applications}

Temporal relations are central in NLP, planning, scheduling, robotics, healthcare, and event recognition. TimeML and TempEval provide annotation and evaluation frameworks for temporal expressions, events, and relations in text \cite{pustejovsky2003timeml,verhagen2007tempeval}. Bayesian chronology tools also compute empirical probabilities of Allen relations from posterior samples of uncertain periods \cite{dye2015sequence,archaeophases2025allen}. In temporal relation extraction, the difficulty of distinguishing fine Allen relations has motivated coarse-to-fine and aggregated label schemes \cite{vashishtha2019finegrained,huang2023unified}, which our taxonomy renders principled as conditional refinement within a parent predicate. These works demonstrate the utility of probabilistic temporal relations, but they do not provide a general analytic Gaussian algebra for all point, point--interval, and interval--interval cases.

Vague temporal adverbials are a particularly relevant application. Words such as \emph{recently}, \emph{just}, \emph{shortly}, and \emph{long ago} specify graded constraints over temporal distance and relation. Recent work has modeled such adverbials probabilistically as distributions conditioned by event type \cite{kenneweg2025factorized}. A distribution-induced Allen algebra provides a mechanism for mapping between such graded linguistic categories and concrete temporal probability distributions.

A central modeling choice is whether these adverbials should be tied to absolute temporal distances or to confidence in qualitative relation families. Absolute-distance accounts are strongly context-sensitive: \emph{shortly before} may mean seconds in one domain and months in another. In our framework, relation probabilities provide a dimensionless intermediate representation. The modifier \emph{shortly} can be associated with a weak but dominant precedence probability, while \emph{long before} can be associated with near-certain precedence, optionally combined with a context-specific expected gap.

\section{Approach}

\subsection{Overview: From the Generative Definition to Closed Forms}
\label{sec:roadmap}

The construction turns the generative definition into closed-form probabilities in four steps (Figure~\ref{fig:pipeline}); the remainder of this section develops each, and we carry one relation, \overlaps, through all four (Section~\ref{sec:interval-interval}).
\begin{enumerate}
\item \emph{Definition.} A relation's probability is the mass of its indicator under the joint density, \(P(X\,R\,Y)=\int \mathbf 1[X\,R\,Y]\,p(X,Y)\dd X\dd Y\) (Equation~\eqref{eq:generative}).
\item \emph{Region.} The indicator is a system of linear inequalities in the four boundary differences \(A,B,G,H\)---one row of Table~\ref{tab:prob-allen-ineq} per relation, with a tolerance \(\tau\) for contact (Section~\ref{sec:interval-interval}).
\item \emph{Substitution.} The boundaries, hence the differences, are affine in a latent Gaussian vector \(\bm U\); the region becomes a polytope in \(\bm U\), and the integral a conditional Gaussian mass \(P(R)=\Prob(L_R(\bm U)\le0,\ \bm D\ge0)/Q\), writing \(\Prob\) for the probability of a latent event (Equation~\eqref{eq:orthant}; the same map underlies any parameterization, Appendix~\ref{app:parameterizations}).
\item \emph{Evaluation.} A Gaussian mass over a polytope is a multivariate-normal CDF \(\Phi_k\) (the \(k\)-variate generalization of the error function), reducing to \(\erfop/\erfcop\) for two points (Section~\ref{sec:point-point}) and to the bivariate \(\Phi_2\) for a point and an interval (Section~\ref{sec:point-interval}).
\end{enumerate}

\begin{figure}[t]
\centering
\begin{tikzpicture}[
  >=Stealth,
  stp/.style={draw, rounded corners, align=center, font=\small, inner sep=4pt,
              text width=27mm, minimum height=16mm}
]
  \node[stp] (def) {\textbf{1.\ Definition}\\[2pt]$P(X\,R\,Y)=\int\mathbf 1[X\,R\,Y]\,p$};
  \node[stp, right=6mm of def] (reg) {\textbf{2.\ Region}\\[2pt]$\mathbf 1[X\,R\,Y]=\{L_R\le0\}$\\in $A,B,G,H$ (Tab.~\ref{tab:prob-allen-ineq})};
  \node[stp, right=6mm of reg] (sub) {\textbf{3.\ Substitution}\\[2pt]polytope in $\bm U$:\\$\Prob(L_R\le0,\bm D\ge0)/Q$};
  \node[stp, right=6mm of sub] (ev) {\textbf{4.\ Evaluation}\\[2pt]Gaussian mass\\$=\Phi_k$ ($\erfop/\erfcop$)};
  \draw[->] (def)--(reg); \draw[->] (reg)--(sub); \draw[->] (sub)--(ev);
\end{tikzpicture}
\caption{From the generative definition to a closed form in four steps, threaded by the running \overlaps\ example (this section; the inequality table is Table~\ref{tab:prob-allen-ineq}, the orthant form Equation~\eqref{eq:orthant}).}
\Description{A four-box flow diagram: generative definition of a relation probability, its region as a set of linear inequalities in the boundary differences, substitution of the interval parameters, and evaluation as a multivariate Gaussian probability.}
\label{fig:pipeline}
\end{figure}
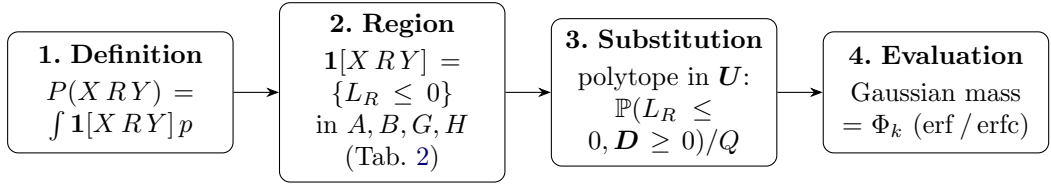

\subsection{Uncertain Times and Durations}
\label{sec:representation}

A time point \(x\) is represented by
\begin{equation}
    t_x\sim \N(\mu_x,\sigma_x^2).
\end{equation}
An interval \(X\) is represented by a midpoint \(t_X\) and non-negative duration \(d_X\):
\begin{equation}
    X=[a_X,b_X]
      =\left[t_X-\frac{d_X}{2},\ t_X+\frac{d_X}{2}\right].
\end{equation}
The midpoint is Gaussian,
\begin{equation}
    t_X\sim\N(\mu_X,\sigma_X^2),
\end{equation}
and the duration is a lower-truncated latent Gaussian,
\begin{equation}
    d_X=D_X\given D_X\ge0,
    \qquad
    D_X\sim\N(\mu_{d_X},\sigma_{d_X}^2).
\end{equation}
For \(\sigma_{d_X}>0\), and writing \(\PhiCDF\) for the standard normal cumulative distribution function, the truncation normalizer is
\begin{equation}
    Q_X=\Prob(D_X\ge0)=\PhiCDF\left(\frac{\mu_{d_X}}{\sigma_{d_X}}\right).
\end{equation}
For \(\sigma_{d_X}=0\), the duration is deterministic and the formula is interpreted as a limit.

We carry two forms of the duration throughout, and the distinction matters below. The lowercase \(d_X\ge0\) is the \emph{duration itself}---the non-negative quantity that enters the boundaries \(a_X,b_X\) and hence the boundary differences \(A,B,G,H\). The uppercase \(D_X\) is its \emph{untruncated latent}, which may be negative but is genuinely Gaussian, so that the vector collecting the latent quantities stays jointly Gaussian and every relation probability is a closed-form multivariate-normal orthant integral (Section~\ref{sec:interval-interval}); non-negativity is imposed by conditioning on \(\{D_X\ge0\}\), with normalizer \(Q_X\). On that admissible event the two coincide, \(d_X=D_X\), so the differences written in \(d\) equal their latent representation in \(D\).

Unless stated otherwise, all primitive Gaussian variables are independent. Correlated temporal quantities can be handled by replacing diagonal covariance matrices with the appropriate joint covariance matrix.

\paragraph{Running example.} We thread one example through the paper. Let \(X\) be an uncertain \emph{storm} and \(Y\) an uncertain \emph{power outage}, in hours, with
\[
X:\ \mu_X=2,\ \sigma_X=0.5,\ \mu_{d_X}=4,\ \sigma_{d_X}=0.5;
\qquad
Y:\ \mu_Y=3,\ \sigma_Y=0.6,\ \mu_{d_Y}=3,\ \sigma_{d_Y}=0.6,
\]
so the storm spans roughly \([0,4]\,\mathrm{h}\) and the outage roughly \([1.5,4.5]\,\mathrm{h}\). A report might say ``the power went out during the storm,'' but the times are uncertain. We return to this example to read off graded relations and their coarse-predicate and primitive structure (Section~\ref{sec:case-study}).

\subsection{Placing the Uncertainty: Which Quantities Are Known}
\label{sec:parameterization}

An interval has two degrees of freedom, so a model is fixed by placing Gaussian uncertainty on \emph{two} of its temporal quantities and deriving the rest. The natural pair is a temporal anchor together with the duration, and which anchor to choose is dictated by what the source actually pins down. Natural language makes the choice vivid: each of the following utterances constrains a different pair and leaves the complementary quantity uncertain.

\begin{itemize}[leftmargin=1.4em]
    \item \emph{Start and duration.} ``I went there on 11 February and stayed for about two months.'' The start is sharp and the duration is vague, so the \emph{end} is the uncertain quantity; model it with a tight standard deviation on the start and a loose one on the duration.
    \item \emph{Midpoint and duration.} ``I was there for ten days during the summer.'' Here the duration is sharp but its \emph{position} is vague---the ten days sit somewhere in summer---so the uncertainty lives on the midpoint. This is the parameterization used throughout the body.
    \item \emph{End and duration.} ``After a week's holiday I got home on Wednesday evening.'' The end is anchored and the duration is roughly a week, so the \emph{start} is what floats.
\end{itemize}

All three describe the same kind of object and produce the same Allen machinery: the duration \(d\ge0\) is truncated as above, the chosen anchor is Gaussian, the boundaries are affine in the two, and every relation remains a conditional Gaussian orthant probability. They differ only in \emph{which} two quantities are taken as independent and Gaussian, and hence in the correlation they induce between midpoint and duration---zero for the midpoint pair, \(+\sigma_d^2/2\) for start--duration, and \(-\sigma_d^2/2\) for end--duration. Appendix~\ref{app:parameterizations} gives the explicit boundary maps, and the accompanying implementation provides dedicated constructors for each interval specification---by start, end, or midpoint together with the duration. The modelling task is therefore only to name the two quantities the data constrains and attach an uncertainty to each; the relation probabilities then follow.

\subsection{A Common Linear-Inequality Semantics}

The central observation is that all Allen relations are conjunctions of linear boundary inequalities. If boundaries are affine functions of Gaussian variables, each relation \(R\) is an event of the form \(\bm A_R\bm U\le\bm b_R\), where \(\bm U\) is a Gaussian vector containing time points and latent durations. Folding the right-hand side into the map, we write this system throughout as
\begin{equation}
    L_R(\bm U)\le\bm 0,\qquad L_R(\bm U)=\bm A_R\bm U-\bm b_R,
    \label{eq:LR-def}
\end{equation}
and use the \(L_R\) form from here on. Its coefficient matrix and offset factor as \(\bm A_R=S_R M\) and \(\bm b_R\): the interval parameterization enters only through the difference map \(M\), the relation only through the signed selector \(S_R\) read off its signature (Table~\ref{tab:primitive-signs}); Appendix~\ref{app:parameterizations} develops this decomposition. If some duration variables are truncated, relation probabilities are conditional Gaussian orthant probabilities:
\begin{equation}
    P(R\given D_i\ge0)
    =\frac{\Prob\big(L_R(\bm U)\le\bm 0,\ D_i\ge0\ \forall i\big)}{\Prob(D_i\ge0\ \forall i)}.
\end{equation}
This formulation is more general than the closed forms below and is the basis for numerical implementation.

\subsection{Relations Between Two Gaussian Points}
\label{sec:point-point}

Let
\begin{equation}
    t_x\sim\N(\mu_x,\sigma_x^2),\qquad
    t_y\sim\N(\mu_y,\sigma_y^2),
\end{equation}
independently. Define
\begin{equation}
    \Delta=\mu_y-\mu_x,
    \qquad
    \sigma_{xy}=\sqrt{\sigma_x^2+\sigma_y^2},
\end{equation}
so that \(Z=t_y-t_x\sim\N(\Delta,\sigma_{xy}^2)\). Then
\begin{equation}
\begin{aligned}
    P(x<y)
    &=\Prob(Z>0) \\
    &=\PhiCDF\left(\frac{\Delta}{\sigma_{xy}}\right) \\
    &=\frac12\left[1+\erfop\left(\frac{\Delta}{\sqrt2\sigma_{xy}}\right)\right]
      =\frac12\erfcop\left(-\frac{\Delta}{\sqrt2\sigma_{xy}}\right),
\end{aligned}
\end{equation}
and
\begin{equation}
\begin{aligned}
    P(x>y)
    &=\Prob(Z<0) \\
    &=\PhiCDF\left(-\frac{\Delta}{\sigma_{xy}}\right) \\
    &=\frac12\left[1-\erfop\left(\frac{\Delta}{\sqrt2\sigma_{xy}}\right)\right]
      =\frac12\erfcop\left(\frac{\Delta}{\sqrt2\sigma_{xy}}\right).
\end{aligned}
\end{equation}
Exact equality has probability zero for non-degenerate continuous Gaussians.

\subsection{Confidence Borders and Graded Membership}
\label{sec:confidence-borders}

Allen relations involving equality, for example \meets, \starts, \finishes, and \equals, occupy lower-dimensional subsets of the continuous configuration space. Their exact probabilities are therefore zero unless deterministic boundary atoms are present. For practical reasoning, and especially for vague linguistic categories, exact equality should be replaced by a confidence border or membership function.

With a hard confidence border \(\tau\ge0\),\footnote{We use the \emph{same} tolerance \(\tau\) for the strict separation margins and for the equality/contact bands. This identity is what makes the thirteen relations a partition whose probabilities sum to one (Proposition~\ref{prop:partition}): each boundary difference \(v\) is split into \(v>\tau\), \(|v|\le\tau\), and \(v<-\tau\), and these three tile the line without gap or overlap only when the band half-width equals the separation margin. A distinct contact tolerance \(\tau_{\mathrm{eq}}\neq\tau\) is admissible but forfeits the partition---a narrower band leaves a gap (total mass \(<1\)), a wider band overlaps (total mass \(>1\)).} approximate equality is
\begin{equation}
    x\approx_{\tau}y
    \quad\Longleftrightarrow\quad
    |t_y-t_x|\le \tau.
\end{equation}
Its probability is
\begin{equation}
\begin{aligned}
    P(x\approx_{\tau}y)
    &=\PhiCDF\left(\frac{\tau-\Delta}{\sigma_{xy}}\right)
      -\PhiCDF\left(\frac{-\tau-\Delta}{\sigma_{xy}}\right) \\
    &=\frac12\left[
      \erfop\left(\frac{\tau-\Delta}{\sqrt2\sigma_{xy}}\right)
      -\erfop\left(\frac{-\tau-\Delta}{\sqrt2\sigma_{xy}}\right)
      \right].
\end{aligned}
\end{equation}
This equality border should be distinguished from the verbal grading of non-contact relations. A tolerance \(\tau\) is required for contact relations because exact boundary equality has zero probability in a continuous model. It is not, however, the primary representation of expressions such as \emph{shortly before}. Those expressions can be modeled by the strength of the coarser precedence predicate relative to overlap and after.

A soft alternative is a membership kernel \(m_{\mathrm{eq}}(z)\in[0,1]\), for example Gaussian or logistic, and the graded truth becomes
\begin{equation}
    \E[m_{\mathrm{eq}}(t_y-t_x)].
\end{equation}
The hard-border formulas are used below because they remain expressible as Gaussian CDFs over linear inequalities.

\subsection{Relations Between a Point and an Interval}
\label{sec:point-interval}

Let \(x\) be a Gaussian point and \(Y\) an uncertain interval. Define
\begin{equation}
    Z=t_Y-t_x,
    \qquad
    \Delta_{Yx}=\mu_Y-\mu_x,
    \qquad
    \sigma_Z^2=\sigma_Y^2+\sigma_x^2.
\end{equation}
The interval boundaries are \(a_Y=t_Y-d_Y/2\) and \(b_Y=t_Y+d_Y/2\). The point is before the interval iff \(t_x<a_Y\), equivalently \(Z>d_Y/2\); it is after iff \(Z<-d_Y/2\); and it is inside iff \(|Z|\le d_Y/2\).

For \(D_Y\sim\N(\mu_{d_Y},\sigma_{d_Y}^2)\) conditioned on \(D_Y\ge0\), define
\begin{equation}
    Q_Y=\PhiCDF\left(\frac{\mu_{d_Y}}{\sigma_{d_Y}}\right),
    \qquad
    \alpha_Y=\frac{\mu_{d_Y}}{\sigma_{d_Y}},
\end{equation}
\begin{equation}
    \sigma_{pI}=\sqrt{\sigma_Z^2+\frac14\sigma_{d_Y}^2},
    \qquad
    \rho_Y=-\frac{\sigma_{d_Y}}{2\sigma_{pI}}.
\end{equation}
Let \(\Phi_2(a,b;\rho)\) denote the standard bivariate normal CDF with correlation \(\rho\). For a separation margin \(\tau\ge0\), define
\begin{equation}
    \beta^{pI}_{+}(\tau)
    =\frac{\Delta_{Yx}-\tau-\frac12\mu_{d_Y}}{\sigma_{pI}},
    \qquad
    \beta^{pI}_{-}(\tau)
    =\frac{-\Delta_{Yx}-\tau-\frac12\mu_{d_Y}}{\sigma_{pI}}.
\end{equation}
Then
\begin{equation}
    P^{pI}_{\mathrm{before}}(\tau)
    =\Prob(t_x<a_Y-\tau)
    =\frac{\Phi_2(\alpha_Y,\beta^{pI}_{+}(\tau);\rho_Y)}{Q_Y},
\end{equation}
\begin{equation}
    P^{pI}_{\mathrm{after}}(\tau)
    =\Prob(t_x>b_Y+\tau)
    =\frac{\Phi_2(\alpha_Y,\beta^{pI}_{-}(\tau);\rho_Y)}{Q_Y}.
\end{equation}
The point-in-interval probability is
\begin{equation}
    P^{pI}_{\mathrm{inside}}
    =1-P^{pI}_{\mathrm{before}}(0)-P^{pI}_{\mathrm{after}}(0).
\end{equation}
Boundary correspondence uses \(\tau\):
\begin{equation}
    \Prob(|t_x-a_Y|\le\tau),
    \qquad
    \Prob(|t_x-b_Y|\le\tau),
\end{equation}
which are again bivariate Gaussian CDF differences because \(t_x-a_Y\) and \(t_x-b_Y\) are affine functions of \((Z,D_Y)\).

\subsection{Relations Between Two Intervals: Distribution-Induced Allen Algebra}
\label{sec:interval-interval}

Let
\begin{equation}
    X=\left[t_X-\frac{d_X}{2},\ t_X+\frac{d_X}{2}\right],
    \qquad
    Y=\left[t_Y-\frac{d_Y}{2},\ t_Y+\frac{d_Y}{2}\right].
\end{equation}
Define
\begin{equation}
    Z=t_Y-t_X,
    \qquad
    \Delta=\mu_Y-\mu_X,
    \qquad
    \sigma_Z^2=\sigma_X^2+\sigma_Y^2.
\end{equation}
The four boundary differences are
\begin{align}
    A &= a_Y-a_X = Z+\frac{d_X-d_Y}{2}, \\
    B &= b_Y-b_X = Z-\frac{d_X-d_Y}{2}, \\
    G &= a_Y-b_X = Z-\frac{d_X+d_Y}{2}, \\
    H &= a_X-b_Y = -Z-\frac{d_X+d_Y}{2}.
\end{align}
Here \(G\) is the gap from the end of \(X\) to the start of \(Y\), and \(H\) is the converse gap.

On the admissible event \(\{D_X\ge0,\ D_Y\ge0\}\) the durations coincide with their untruncated latents, \(d_X=D_X\) and \(d_Y=D_Y\), so the four differences are affine in the latent vector \(\bm U=(Z,D_X,D_Y)\transpose\), which is jointly Gaussian with mean \(\bm\mu_U=(\Delta,\mu_{d_X},\mu_{d_Y})\) and, in the independent case, covariance \(\bm\Sigma_U=\operatorname{diag}(\sigma_Z^2,\sigma_{d_X}^2,\sigma_{d_Y}^2)\).

The system \(L_R(\bm U)\le0\) of a relation is read directly from Table~\ref{tab:prob-allen-ineq}: put each inequality listed there in the form \({\le}\,0\) and expand it through the differences \(A,B,G,H\) above. A strict margin such as \(G>\tau\) contributes a single affine row \(\tau-G\le0\); a contact band such as \(|A|\le\tau\) contributes two rows, \(A-\tau\le0\) and \(-A-\tau\le0\). For \overlaps\ (the row \(A>\tau,\ B>\tau,\ G<-\tau\) of Table~\ref{tab:prob-allen-ineq}),
\begin{equation}
  L_{\overlaps}(\bm U)=
  \begin{pmatrix}\tau-A\\[1pt]\tau-B\\[1pt]G+\tau\end{pmatrix}
  =\begin{pmatrix}\tau-Z-\tfrac12(D_X-D_Y)\\[2pt]\tau-Z+\tfrac12(D_X-D_Y)\\[2pt]Z-\tfrac12(D_X+D_Y)+\tau\end{pmatrix}\le\bm 0 ,
\end{equation}
each row affine in \(\bm U\); adjoining the truncation rows \(-D_X\le0\) and \(-D_Y\le0\) completes the system \(L_R(\bm U)=\bm A_R\bm U-\bm b_R\le\bm0\) of Equation~\eqref{eq:LR-def}. The relation's probability is the Gaussian mass of that region,
\begin{equation}
    P(R)=\frac{\Prob(L_R(\bm U)\le0,\ D_X\ge0,\ D_Y\ge0)}{\Prob(D_X\ge0)\Prob(D_Y\ge0)}.
    \label{eq:orthant}
\end{equation}
The numerator is a multivariate Gaussian CDF, obtained by standardizing the stacked rows by their induced mean \(\bm A_R\bm\mu_U-\bm b_R\) and covariance \(\bm A_R\bm\Sigma_U\bm A_R\transpose\). Reading a row of Table~\ref{tab:prob-allen-ineq}, substituting the differences \(A,B,G,H\), and standardizing is thus the complete recipe that reproduces the concrete closed form of any of the thirteen relations. In one or two active dimensions this CDF is an ordinary error function: \before\ and \after\ reduce to a single tail \(\tfrac12\erfcop(\cdot)\) (Section~\ref{sec:point-point}) and the point--interval relations to a bivariate \(\Phi_2\) (Section~\ref{sec:point-interval}); the running example \overlaps, whose three rows \(\{\tau-A,\ \tau-B,\ G+\tau\}\le0\) are active, evaluates to the trivariate normal probability \(\Phi_3\) of those standardized rows. Equation~\eqref{eq:generative}, Table~\ref{tab:prob-allen-ineq}, Equation~\eqref{eq:orthant}, and this evaluation are precisely the four steps of the roadmap (Section~\ref{sec:roadmap}).

\subsubsection{The Tolerance Partition}
\label{sec:partition}

A single tolerance \(\tau\ge0\) governs both the separation margins of the full-dimensional relations and the half-width of the equality/contact bands. Introduce the thresholded sign
\begin{equation}
    \operatorname{sgn}_\tau(v)=
    \begin{cases}
        +1 & v>\tau,\\
        \phantom{+}0 & |v|\le\tau,\\
        -1 & v<-\tau.
    \end{cases}
\end{equation}
The pair \((\operatorname{sgn}_\tau A,\operatorname{sgn}_\tau B)\) partitions the boundary-difference plane into nine cells; the two diagonal corners are split by \(\operatorname{sgn}_\tau G\) and \(\operatorname{sgn}_\tau H\) into separation/contact triples. Table~\ref{tab:prob-allen-ineq} lists the thirteen resulting inequality systems.

\begin{longtable}{lll}
\caption{The thirteen relations as a single-tolerance partition in boundary differences.}\label{tab:prob-allen-ineq}\\
\toprule
Relation & Inequality form & Interpretation \\
\midrule
\endfirsthead
\toprule
Relation & Inequality form & Interpretation \\
\midrule
\endhead
\before & \(G>\tau\) & \(X\) ends before \(Y\) starts \\
\meets & \(A>\tau,\ B>\tau,\ |G|\le\tau\) & external contact from \(X\) to \(Y\) \\
\overlaps & \(A>\tau,\ B>\tau,\ G<-\tau\) & \(X\) starts before \(Y\) and ends inside \(Y\) \\
\starts & \(|A|\le\tau,\ B>\tau\) & shared start, \(X\) shorter \\
\during & \(A<-\tau,\ B>\tau\) & \(X\) properly inside \(Y\) \\
\finishes & \(A<-\tau,\ |B|\le\tau\) & shared finish, \(X\) shorter \\
\equals & \(|A|\le\tau,\ |B|\le\tau\) & shared start and finish \\
\finishedby & \(A>\tau,\ |B|\le\tau\) & shared finish, \(Y\) shorter \\
\contains & \(A>\tau,\ B<-\tau\) & \(Y\) properly inside \(X\) \\
\startedby & \(|A|\le\tau,\ B<-\tau\) & shared start, \(Y\) shorter \\
\overlappedby & \(A<-\tau,\ B<-\tau,\ H<-\tau\) & \(Y\) starts before \(X\) and ends inside \(X\) \\
\metby & \(A<-\tau,\ B<-\tau,\ |H|\le\tau\) & external contact from \(Y\) to \(X\) \\
\after & \(H>\tau\) & \(X\) starts after \(Y\) ends \\
\bottomrule
\end{longtable}

\begin{proposition}[Tolerance partition]
\label{prop:partition}
For every \(\tau\ge0\) and every joint law of the boundaries with \(a_X\le b_X\) and \(a_Y\le b_Y\) almost surely, the thirteen relations of Table~\ref{tab:prob-allen-ineq} are pairwise disjoint and cover the sample space. Hence their probabilities sum to one, and as \(\tau\to0^+\) they converge to the crisp Allen relations, with the seven equality/contact relations having probability zero under continuous boundary distributions.
\end{proposition}

\begin{proof}
The maps \(a=\operatorname{sgn}_\tau A\) and \(b=\operatorname{sgn}_\tau B\) are single-valued and exhaustive, so the nine cells \(\{(a,b)\}\) partition the sample space. The seven cells with \((a,b)\notin\{(+,+),(-,-)\}\) are each labelled by exactly one relation. On the cell \((+,+)\) the single-valued map \(\operatorname{sgn}_\tau G\) yields the disjoint cover \(\{\overlaps,\allowbreak\meets,\allowbreak\before\}\); symmetrically \(\operatorname{sgn}_\tau H\) splits \((-,-)\) into \(\{\overlappedby,\allowbreak\metby,\allowbreak\after\}\). A disjoint cover of each cell of a partition is again a partition, so the thirteen relations partition the space. The limiting claim follows because each band \(|\cdot|\le\tau\) shrinks to a measure-zero hyperplane as \(\tau\to0\).
\end{proof}

A single tolerance is a requirement, not a convenience: the equality bands tile with the strict regions only when the band half-width equals the separation margin. Two independent tolerances reopen gaps or overlaps and break the partition. We verify numerically that the analytic probabilities sum to one within Monte-Carlo precision for every \(\tau\) (Section~\ref{sec:validation}).

\subsection{Taxonomy as Probability Calculus}
\label{sec:taxonomy}

A central claim of this paper is that Allen relations should not be treated as independent flat classes. They are leaves of a relation taxonomy. Let \(\mathcal L\) be the set of Allen leaf predicates. For any coarse predicate \(C\) defined as a union of leaves,
\begin{equation}
    C=\bigcup_{R\in\mathcal L_C} R,
\end{equation}
and for disjoint leaves,
\begin{equation}
    P(C)=\sum_{R\in\mathcal L_C}P(R).
    \label{eq:coarse-sum}
\end{equation}
This is not merely a post-processing convention; it follows from the fact that all relations are regions of the same probability space.

For example, precedence with contact is
\begin{equation}
    P(X\preceq Y)=P(\before)+P(\meets).
\end{equation}
Containment of \(X\) in \(Y\) is
\begin{equation}
    P(X\subseteq Y)
    =P(\starts)+P(\during)+P(\finishes)+P(\equals).
\end{equation}
Reverse containment is
\begin{equation}
    P(X\supseteq Y)
    =P(\startedby)+P(\contains)+P(\finishedby)+P(\equals).
\end{equation}
Partial overlap is
\begin{equation}
    P(\mathrm{partial\ overlap})=P(\overlaps)+P(\overlappedby).
\end{equation}
External adjacency is
\begin{equation}
    P(\mathrm{outer\ contact})=P(\meets)+P(\metby),
\end{equation}
and inner adjacency is
\begin{equation}
\begin{aligned}
    P(\mathrm{inner\ contact})={}&P(\starts)+P(\startedby)+P(\finishes)\\
    &+P(\finishedby)+P(\equals).
\end{aligned}
\end{equation}

This enables top-down refinement. For instance,
\begin{equation}
    P(\meets\given X\preceq Y)
    =\frac{P(\meets)}{P(\before)+P(\meets)},
\end{equation}
and
\begin{equation}
    P(\starts\given X\subseteq Y)
    =\frac{P(\starts)}{P(\starts)+P(\during)+P(\finishes)+P(\equals)}.
\end{equation}
A flat classifier over thirteen labels cannot express these inheritance relations unless the taxonomy is added externally.

\begin{figure}[htbp]
\centering
\begin{forest}
  for tree={font=\small, grow'=0, folder, l sep=6mm, s sep=1.4mm, anchor=west}
  [relation between $X$ and $Y$
    [separated / externally touching
      [$X$ precedes / touches $Y$
        [{\ipairf{0}{1.5}~\before}]
        [{\ipairf{0.5}{2}~\meets}]
      ]
      [$X$ follows / touched by $Y$
        [{\ipairf{4}{5.5}~\metby}]
        [{\ipairf{4.5}{6}~\after}]
      ]
    ]
    [non-separated
      [partial overlap
        [{\ipairf{1}{3}~\overlaps}]
        [{\ipairf{3}{5}~\overlappedby}]
      ]
      [$X \subset Y$ (strict)
        [{\ipairf{2}{3}~\starts}]
        [{\ipairf{2.5}{3.5}~\during}]
        [{\ipairf{3}{4}~\finishes}]
      ]
      [$X \supset Y$ (strict)
        [{\ipairf{2}{5}~\startedby}]
        [{\ipairf{1}{5}~\contains}]
        [{\ipairf{1}{4}~\finishedby}]
      ]
      [{\ipairf{2}{4}~\equals}]
    ]
  ]
\end{forest}
\caption{Relation taxonomy as a strict partition tree. Each leaf shows the relation with the reference interval $Y$ (outlined) held fixed and the event $X$ (filled) placed accordingly---dotted ticks mark $Y$'s edges, so the leaves are directly comparable (compare the crisp boundary orderings of Table~\ref{tab:allen-relations}). The children of every node are mutually exclusive and exhaustive, so each node's probability is the sum of its children's. \equals\ is its own cell rather than being shared by the two containment directions; the non-strict views $X\subseteq Y$ and $X\supseteq Y$ (each of which also includes \equals) are overlapping coarse predicates, not nodes of this tree.}
\Description{A tree whose root is the set of all thirteen Allen relations, split into coarse predicates (precedes, overlaps, contains, follows and their contact variants) down to the thirteen leaves; each leaf shows a small two-bar sketch of X against a fixed Y.}
\label{fig:relation-taxonomy}
\end{figure}
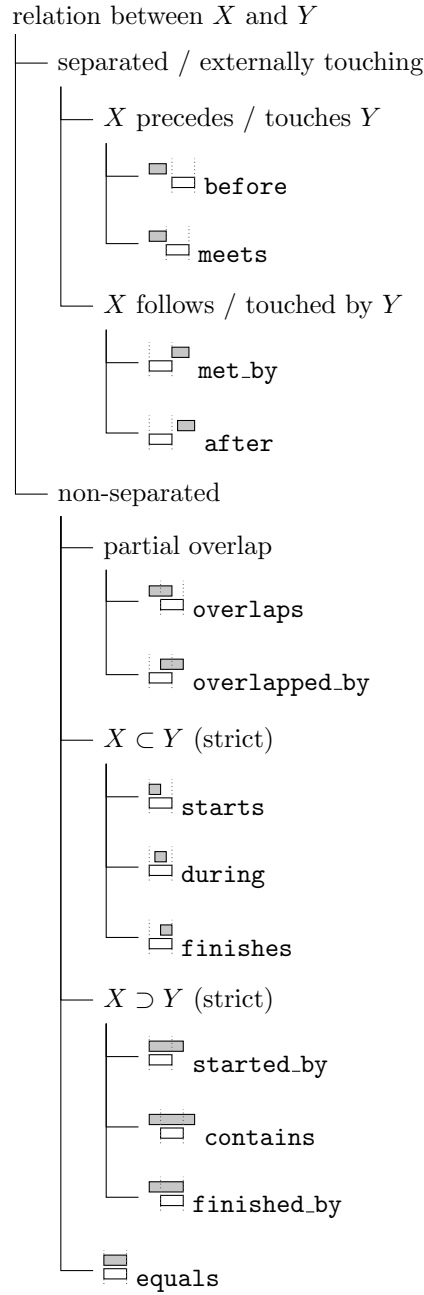

\subsection{Temporal Primitives and the (De)composition of Relations}
\label{sec:primitives}

The thirteen relations are not atomic. Each is a conjunction of elementary assertions about the order of two interval boundaries. With \(X=[a_X,b_X]\) and \(Y=[a_Y,b_Y]\), the within-interval orders \(a_X<b_X\) and \(a_Y<b_Y\) are fixed, so a relation can only constrain the \emph{cross} comparisons of one interval's boundaries against the other's. We call these comparisons \emph{temporal primitives}. They are a general logical foundation for interval relations: the relation algebra is generated by the order type of the boundaries, and every relation is recovered as a conjunction of primitive orderings. The CIDOC~CRM ontology adopts exactly this stance, replacing Allen relations by temporal primitives such as ``starts before the start of'' so that statements stay individually attestable and amenable to reasoning~\cite{doerr2003cidoc}; for instance \(\overlaps\) is recorded as \(a_X\prec a_Y\), \(b_X\succ a_Y\), and \(b_X\prec b_Y\).

Our partition realizes the same foundation in metric, probabilistic form. The four primitives are precisely the thresholded signs of the boundary differences \(A,B,G,H\) of Section~\ref{sec:interval-interval},
\begin{equation}
    \operatorname{sgn}_\tau A,\quad
    \operatorname{sgn}_\tau B,\quad
    \operatorname{sgn}_\tau G,\quad
    \operatorname{sgn}_\tau H,
\end{equation}
comparing \(\mathrm{start}_X/\mathrm{start}_Y\), \(\mathrm{end}_X/\mathrm{end}_Y\), \(\mathrm{end}_X/\mathrm{start}_Y\), and \(\mathrm{end}_Y/\mathrm{start}_X\) respectively. Each is \emph{three-valued}: the strict states \(\pm1\) are CRM's qualitative ``before''/``after'', while the middle state \(0\) (the band \(|\cdot|\le\tau\)) is the metric predicate ``the two boundaries coincide within \(\tau\)'' --- the graded notion of contact that crisp primitives cannot express.

Composition is then read directly off the partition. Table~\ref{tab:primitive-signs} lists the primitive signature \((\operatorname{sgn}_\tau A,\allowbreak\operatorname{sgn}_\tau B,\allowbreak\operatorname{sgn}_\tau G,\allowbreak\operatorname{sgn}_\tau H)\) of every relation. The contact relations are exactly those with a coincident primitive. As a worked example, \meets\ has signature \((+,+,0,-)\),
\begin{equation}
    \meets \iff a_X\prec a_Y \ \wedge\ b_X\prec b_Y \ \wedge\ b_X \approx_\tau a_Y,
\end{equation}
where \(\approx_\tau\) denotes coincidence within \(\tau\). It shares the primitives \(A,B,H\) with \before\ and \overlaps\ and differs from them only in \(\operatorname{sgn}_\tau G\), which runs through \(+,0,-\); thus \meets\ is \before\ with its single discriminating primitive flipped from strict precedence to contact --- not the conjunction of \before\ with a contact predicate, which would be contradictory.

\begin{table}[h]
\centering
\caption{Primitive signature \((\operatorname{sgn}_\tau A,\operatorname{sgn}_\tau B,\operatorname{sgn}_\tau G,\operatorname{sgn}_\tau H)\) of each relation and its number of coincident primitives \(c\). Strict relations have \(c=0\); the six touching/aligning relations \(c=1\); \equals\ has \(c=2\).}
\label{tab:primitive-signs}
\begin{tabular}{lccccc}
\toprule
Relation & \(A\) & \(B\) & \(G\) & \(H\) & \(c\) \\
\midrule
\before        & \(+\) & \(+\) & \(+\) & \(-\) & 0 \\
\meets         & \(+\) & \(+\) & \(0\) & \(-\) & 1 \\
\overlaps      & \(+\) & \(+\) & \(-\) & \(-\) & 0 \\
\starts        & \(0\) & \(+\) & \(-\) & \(-\) & 1 \\
\during        & \(-\) & \(+\) & \(-\) & \(-\) & 0 \\
\finishes      & \(-\) & \(0\) & \(-\) & \(-\) & 1 \\
\equals        & \(0\) & \(0\) & \(-\) & \(-\) & 2 \\
\finishedby    & \(+\) & \(0\) & \(-\) & \(-\) & 1 \\
\contains      & \(+\) & \(-\) & \(-\) & \(-\) & 0 \\
\startedby     & \(0\) & \(-\) & \(-\) & \(-\) & 1 \\
\overlappedby  & \(-\) & \(-\) & \(-\) & \(-\) & 0 \\
\metby         & \(-\) & \(-\) & \(-\) & \(0\) & 1 \\
\after         & \(-\) & \(-\) & \(-\) & \(+\) & 0 \\
\bottomrule
\end{tabular}
\end{table}

Read as inequalities, each signature is the relation's selector \(S_R\)---a matrix with entries in \(\{-1,0,+1\}\) acting on \((A,B,G,H)\transpose\) (Appendix~\ref{app:parameterizations}): a strict sign \(\sigma_v\) contributes the row \(-\sigma_v v\le-\tau\) and a coincidence \(\sigma_v=0\) the pair \(\pm v\le\tau\). For the six strict relations \(S_R=-\operatorname{diag}(\sigma_R)\); the contact relations add a \(\pm\) row per coincident primitive. Each relation's linear inequality system is thus read directly off this table.

The coincident-primitive count characterizes equality and contact exactly.

\begin{proposition}[Equality and contact as coincident primitives]
\label{prop:contact}
Let \(c(R)=\#\{v\in\{A,B,G,H\}:\operatorname{sgn}_\tau v=0\}\) be the number of coincident primitives of a relation \(R\). Then \(c\) equals the codimension of \(R\)'s defining set in boundary-difference space: the six relations \(\{\before,\allowbreak\overlaps,\allowbreak\during,\allowbreak\contains,\allowbreak\overlappedby,\allowbreak\after\}\) have \(c=0\); the six touching/aligning relations \(\{\meets,\allowbreak\starts,\allowbreak\finishes,\allowbreak\finishedby,\allowbreak\startedby,\allowbreak\metby\}\) have \(c=1\); and \(\equals\) has \(c=2\). Consequently \(R\) retains positive probability under every continuous boundary law in the crisp limit \(\tau\to0^+\) if and only if \(c(R)=0\); the seven relations with \(c\ge1\) are precisely the equality/contact relations of Proposition~\ref{prop:partition}, whose mass is carried entirely by bands of width \(O(\tau)\).
\end{proposition}

\begin{proof}
From Table~\ref{tab:prob-allen-ineq}, the coincidence state \(\operatorname{sgn}_\tau{=}0\) occurs for \(A\) on \(\{\starts,\allowbreak\equals,\allowbreak\startedby\}\), for \(B\) on \(\{\finishes,\allowbreak\equals,\allowbreak\finishedby\}\), for \(G\) on \(\{\meets\}\), and for \(H\) on \(\{\metby\}\). Collecting by relation gives \(c=2\) for \equals, \(c=1\) for the other six listed, and \(c=0\) for the remaining six, which is the stated trichotomy. Each coincident primitive constrains its difference to the band \(|v|\le\tau\), whose \(\tau\to0\) limit is the hyperplane \(\{v=0\}\) of codimension one; the coincident primitives of each relation are linearly independent (for \equals, \(A\) and \(B\) span two dimensions; the others impose a single constraint), so \(c\) of them cut out a set of codimension \(c\). A set of positive codimension is Lebesgue-null and carries zero probability under any continuous law, giving the equivalence; for \(\tau>0\) each band has width \(2\tau\) in its difference coordinate, so the contact mass is \(O(\tau)\).
\end{proof}

This factorization holds for the whole algebra at once.

\begin{proposition}[Primitive factorization]
\label{prop:factorization}
Let \(\Phi(\bm U)=(\operatorname{sgn}_\tau A,\operatorname{sgn}_\tau B,\operatorname{sgn}_\tau G,\operatorname{sgn}_\tau H)\) with \(\bm U=(Z,D_X,D_Y)\transpose\). On the feasible cone \(\{D_X,D_Y\ge0\}\) each relation is a primitive fiber, \(\{R\}=\Phi^{-1}(\sigma_R)\), where \(\sigma_R\) is the signature of Table~\ref{tab:primitive-signs}. Hence every relation probability is the joint law of the four primitives at a single atom,
\begin{equation}
\label{eq:primitive-joint}
    P(R)=P\!\Big(\textstyle\bigcap_{v\in\{A,B,G,H\}}\{\operatorname{sgn}_\tau v=\sigma_R(v)\}\,\Big|\,D_X,D_Y\ge0\Big),
\end{equation}
a single multivariate-Gaussian band--orthant integral; and the marginal of each primitive is the coarse predicate \(P(\operatorname{sgn}_\tau v=s)=\sum_{R:\,\sigma_R(v)=s}P(R)\), a one-dimensional Gaussian tail or band. Since the four differences span only three dimensions (\(H=G-A-B\)), the joint of \eqref{eq:primitive-joint} is not the product of these marginals.
\end{proposition}

Finally, the decomposition has two probabilistic levels. Marginally, each primitive is a proper three-way distribution---its states partition the corresponding axis---and its probabilities are sums of leaf probabilities, i.e. coarse predicates of the taxonomy (Section~\ref{sec:taxonomy}). The start/start primitive, for example, induces the coarsening
\begin{equation}
\begin{aligned}
    P(\operatorname{sgn}_\tau A={+}) &= P(\before)+P(\meets)+P(\overlaps)+P(\finishedby)+P(\contains),\\
    P(\operatorname{sgn}_\tau A={0}) &= P(\starts)+P(\equals)+P(\startedby),\\
    P(\operatorname{sgn}_\tau A={-}) &= P(\during)+P(\finishes)+P(\overlappedby)+P(\metby)+P(\after),
\end{aligned}
\end{equation}
whose middle cluster is precisely ``shared start''; the four primitives are thus four three-way coarsenings of the leaf set, dual to the taxonomy's nodes. A relation, by contrast, is the \emph{joint} law of its primitives. Because all four differences are linear in the shared latent \(\bm U=(Z,D_X,D_Y)\) (indeed \(H=G-A-B\)), the primitives are strongly correlated and a relation is never the product of its primitive marginals: the gap primitive alone entails the leading-edge orders, since \(G>\tau\Rightarrow A,B>\tau\) and hence \(P(\before)=P(\operatorname{sgn}_\tau G={+})\), whereas at a shared midpoint \(A\approx-B\) precludes both from being large. The derivation, with the one-dimensional closed forms for the primitive marginals, is given in Appendix~\ref{app:factorization}. This is what the probabilistic lift adds to the CRM foundation: each primitive remains an independently attestable, now graded, statement, while a single covariance over \(\bm U\) supplies the dependence needed to compose them into coherent relation probabilities. On the running example this localization is striking: the entire thirteen-way distribution reduces to uncertainty in a single primitive (Section~\ref{sec:case-study}).

\subsection{Consistency of Limiting Cases}

The taxonomy makes boundary limits semantically stable. Consider the gap
\begin{equation}
    g=a_Y-b_X.
\end{equation}
In crisp Allen algebra, \(g>0\) gives \before\ and \(g=0\) gives \meets. A flat model treats this as a class switch. In the taxonomy, both are children of the coarser predicate \(X\preceq Y\), and by the partition of Section~\ref{sec:partition} this parent probability is their sum,
\begin{equation}
    P(X\preceq Y)=P(\before)+P(\meets),
\end{equation}
which is stable even though the leaf mass migrates between \before\ (the region \(g>\tau\)) and \meets\ (the band \(|g|\le\tau\)) as the tolerance or the expected gap varies.
As the expected gap approaches zero, probability mass moves from strict \before\ to approximate \meets, while the parent probability remains well-defined. Analogously, \during\ can approach \starts\ or \finishes\ as one boundary aligns with the container, while containment remains stable:
\begin{equation}
    P(X\subseteq Y)=P(\starts)+P(\during)+P(\finishes)+P(\equals).
\end{equation}
This is a key robustness advantage over flat multiclass relation assignment.

\section{Validation}
\label{sec:validation}

The accompanying open-source implementation computes the multivariate Gaussian orthant probabilities and, independently, a Monte-Carlo estimate that samples boundaries and classifies each draw by Table~\ref{tab:prob-allen-ineq}. All claims below are reproduced by that package.

\paragraph{Partition.} The thirteen probabilities sum to one for every tolerance \(\tau\), confirming Proposition~\ref{prop:partition}; the Monte-Carlo classifier assigns every sample to exactly one relation. By contrast, a naive one-sided contact band (assigning \meets\ to \(0\le G\le\tau\) rather than \(|G|\le\tau\)) loses several percent of the probability mass into uncovered buffer zones.

\paragraph{Agreement.} The analytic probabilities match the Monte-Carlo frequencies to within sampling noise across all thirteen relations: the maximum absolute deviation is about \(2\times10^{-4}\) at \(4\times10^{6}\) samples.

\paragraph{Graded transitions and limits.} As one interval slides past the other, probability mass transfers smoothly between relations (Figure~\ref{fig:sweep}). As a duration shrinks to zero, the thirteen relations reduce continuously to the five point--interval relations \(\{\before,\allowbreak\starts,\allowbreak\during,\allowbreak\finishes,\allowbreak\after\}\) (Figure~\ref{fig:limit}), and to \(\{\before,\allowbreak\equals,\allowbreak\after\}\) for two points, with the partition preserved at every step.

\paragraph{Scale invariance.} Scaling all means, durations, standard deviations, and the tolerance by a common factor leaves every relation probability unchanged, as the standardized arguments are invariant.

\begin{figure}[t]
\centering
\includegraphics[width=0.82\linewidth]{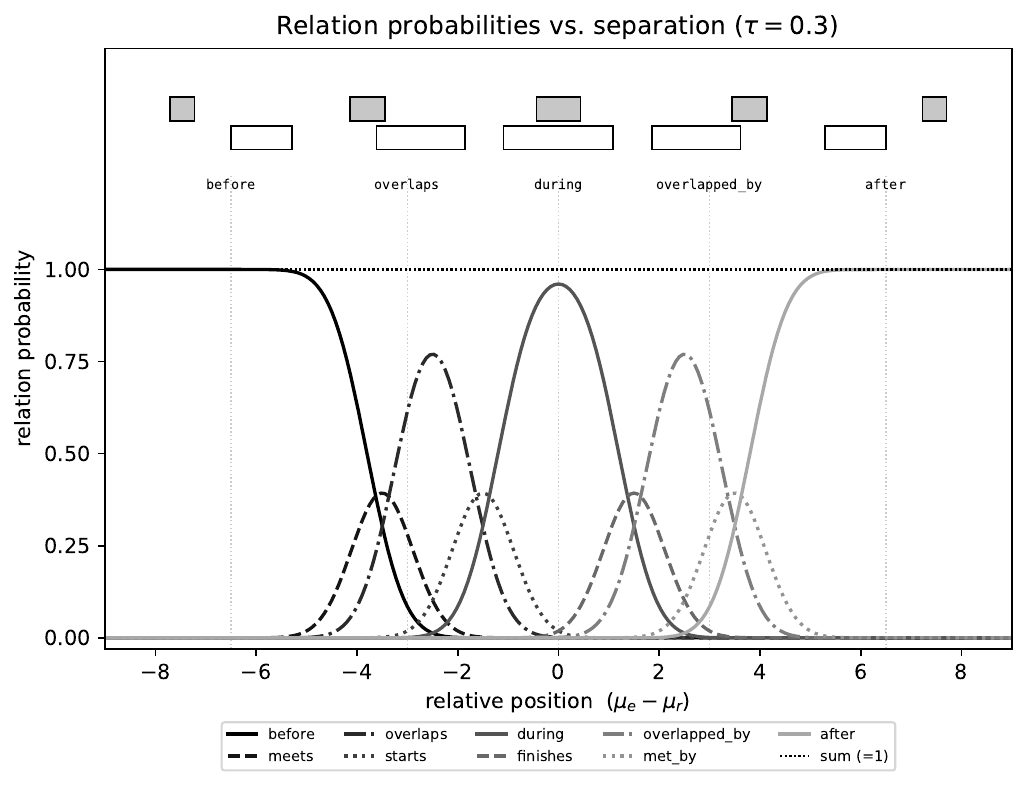}
\caption{Relation probabilities as one interval slides across the other. Mass transfers smoothly through the relations while the curves sum to one everywhere.}
\Description{Line plot of the thirteen relation probabilities against the position of one interval sliding across the other; the curves rise and fall in the Allen order and sum to one at every position.}
\label{fig:sweep}
\end{figure}

\begin{figure}[t]
\centering
\includegraphics[width=0.72\linewidth]{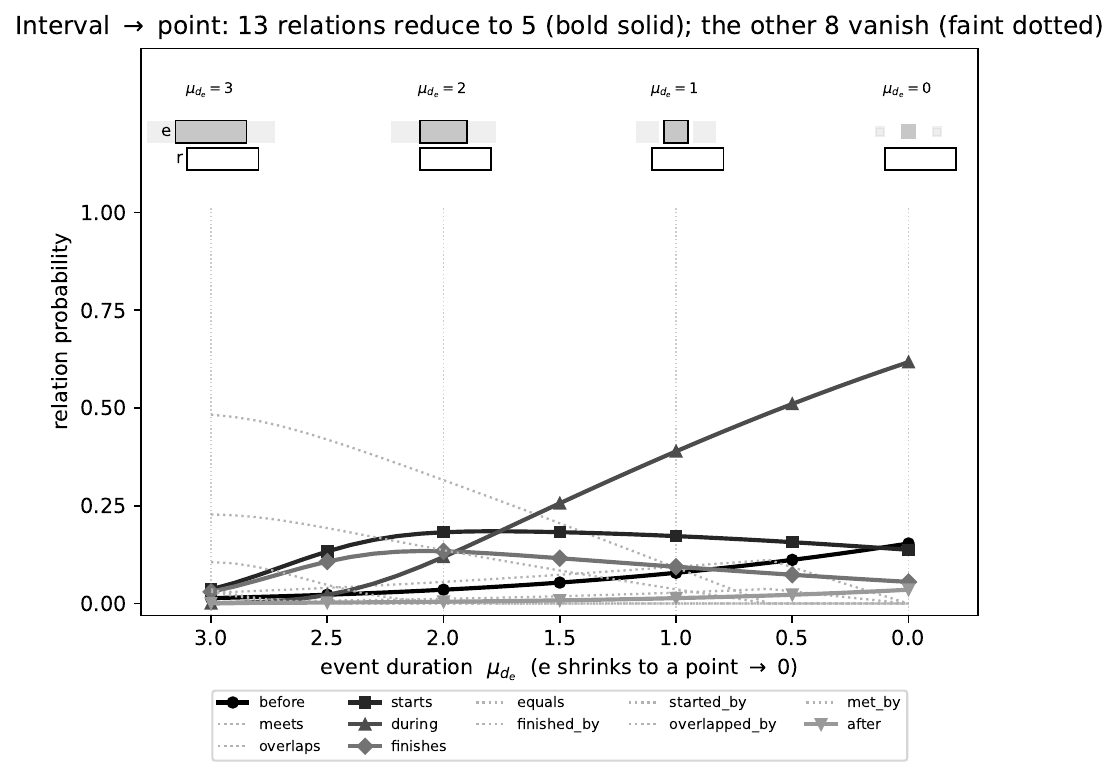}
\caption{Interval-to-point limit: as the event duration \(\mu_{d_e}\to0\) (right edge), the thirteen interval--interval relations reduce continuously to the five point--interval relations \(\{\before,\allowbreak\starts,\allowbreak\during,\allowbreak\finishes,\allowbreak\after\}\) (bold solid); the others lose their mass and vanish (faint dotted). The top panels show the constellation at four durations: the event \(e\) (filled, with faint copies for its positional uncertainty) shrinks against the fixed reference \(r\) (outlined) down to a point, offset from \(r\)'s midpoint so the \starts/\finishes\ and \before/\after\ curves separate. The thirteen probabilities sum to one throughout.}
\Description{Stacked-area plot of the thirteen relation probabilities as the event duration shrinks to zero; eight relations vanish and the mass concentrates on the five point-interval relations.}
\label{fig:limit}
\end{figure}

\subsection{A Worked Example}
\label{sec:case-study}

We collect the threads on the running storm/outage example of Section~\ref{sec:representation}. Figure~\ref{fig:case-study} shows the two uncertain intervals and the resulting relation probabilities at \(\tau=0.4\,\mathrm{h}\).

\paragraph{Graded relations.} A crisp classifier sees only the mode, \overlaps\ (\(P\approx0.52\)): the storm starts first and the outage runs past its end. But the distribution is broad---\finishedby\ (\(0.27\)), \contains\ (\(0.10\)), \startedby\ and \equals\ (a few percent each)---so the single label discards most of the picture.

\paragraph{Taxonomy.} The coarse predicates make the real question legible: the outage spilling past the storm and the outage being contained within it are almost equally likely, \(P(\text{partial overlap})\approx0.53\) against \(P(Y\subseteq X)\approx0.44\), while precedence is negligible (\(\approx0.01\)). Neither coarse probability is a separate bar in Figure~\ref{fig:case-study}; each is the sum of its leaves there, following the taxonomy rule that a parent's probability is the total of its children's. Containment is the sum of its four leaves,
\begin{align*}
    P(Y\subseteq X)&=P(\finishedby)+P(\contains)+P(\startedby)+P(\equals)\\
    &\approx0.27+0.10+0.04+0.03\approx0.44,
\end{align*}
and partial overlap is \(P(\overlaps)+P(\overlappedby)\). The graded model reports this near-even split; a flat \overlaps\ label cannot.

\paragraph{Primitive localization.} The relational uncertainty collapses onto a single temporal primitive. The storm almost surely starts first, \(P(\operatorname{sgn}_\tau A={+})\approx0.90\) (``the outage began after the storm started''), and the two certainly overlap: \(\operatorname{sgn}_\tau G\) and \(\operatorname{sgn}_\tau H\) are \emph{determinate}---each concentrated on a single sign with probability \(\approx1\)---since the outage almost surely begins before the storm ends (\(\operatorname{sgn}_\tau G={-}\)) and the storm before the outage ends (\(\operatorname{sgn}_\tau H={-}\)), so neither a leading nor a trailing gap can open. What is undecided is whether the storm ends before, together with, or after the outage---primitive \(B\)---with probabilities \(\approx0.55,\,0.30,\,0.15\). The thirteen-way uncertainty is, at heart, one attestable boundary comparison.

\paragraph{Limits and scale.} If the outage is reported as an instant rather than an interval (duration \(\to0\)), the relations reduce to the point--interval set, now dominated by \contains\ (\(\approx0.77\))---the blackout instant falling within the storm. Expressing the same scenario in minutes rather than hours (scaling means, durations, spreads, and \(\tau\) by \(60\)) leaves every probability unchanged to \(2\times10^{-5}\).

\paragraph{Language.} The report's ``during the storm'' is, quantitatively, a roughly even mixture of containment and trailing overlap; a sharper adverbial such as ``well into the storm'' would concentrate \(\operatorname{sgn}_\tau A={+}\) and shift mass from \starts\ and \equals\ toward \during\ and \contains. The algebra turns vague temporal language into a distribution over boundary orderings.

\begin{figure}[t]
\centering
\includegraphics[width=0.86\linewidth]{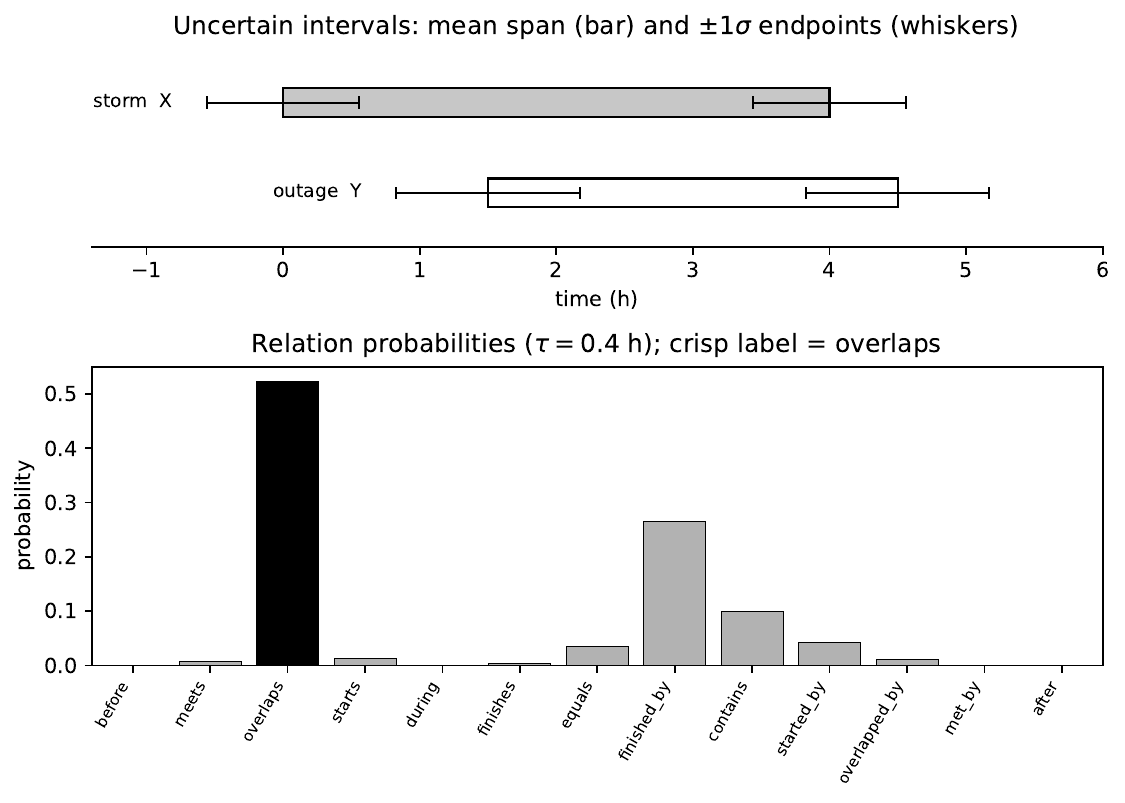}
\caption{Running example. Top: the uncertain storm \(X\) and power outage \(Y\) (mean span as a bar, \(\pm1\sigma\) boundary uncertainty as whiskers). Bottom: the thirteen relation probabilities at \(\tau=0.4\,\mathrm{h}\); the crisp mode \overlaps\ (black) accounts for only about half the mass.}
\Description{Top: two horizontal bars for the storm and the power outage with whiskers marking one standard deviation of boundary uncertainty. Bottom: bar chart of the thirteen relation probabilities, dominated by overlaps and finished-by.}
\label{fig:case-study}
\end{figure}

\section{Discussion, Conclusions, and Applications}

\subsection{Comparison to Allen's Hourglass}

Allen's Hourglass and the present model share a motivation: crisp Allen relations are brittle under uncertain interval boundaries \cite{petridis2010hourglass}. The hourglass representation organizes relations by relative position and relative size, thereby producing a robust graded treatment of relation assignment. It also makes limiting cases such as durationless intervals geometrically visible.

The present approach differs in three respects. First, it gives a generative semantics: relation probabilities are induced by probability distributions over time points, midpoints, and durations. Second, equality and contact are handled explicitly through confidence borders, avoiding the zero-measure problem of exact boundary equalities in continuous spaces. Third, it preserves the relation taxonomy as probability calculus rather than treating the thirteen Allen labels as independent classes.

Thus, Allen's Hourglass can be seen as a conceptual geometry of robust relation assignment, while the present model provides a distributional semantics and analytic computation of relation probabilities. The two approaches are complementary: the hourglass is valuable for visualization and intuition; the Gaussian inequality formulation is valuable for inference, limiting cases, and integration with metric uncertainty.

\subsection{The Relation Phase Diagram, Taxonomy, and Dimensional Reduction}
\label{sec:phase-diagram}

The hourglass is usefully read as a \emph{phase diagram} over relation classes: a stratification of configuration space by the codimension of the defining equalities. The six relations defined by strict inequalities are full-dimensional ``bulk phases''; the six contact relations are codimension-one ``phase boundaries''; and \equals, defined by two equalities, is a codimension-two multi-phase point. Our contribution is to place a \emph{generative measure} on this geometry: a relation probability is the mass that the boundary density assigns to a phase, the tolerance \(\tau\) is the finite width of the coexistence bands, and (as future work) the gradients of relation probabilities are response functions peaked at the boundaries. Computing the most probable relation over this plane reproduces the hourglass geometry directly from the generative model (Figure~\ref{fig:phase}, left), beside the dimensional reduction discussed below (right).

\begin{figure}[t]
\centering
\includegraphics[width=\linewidth]{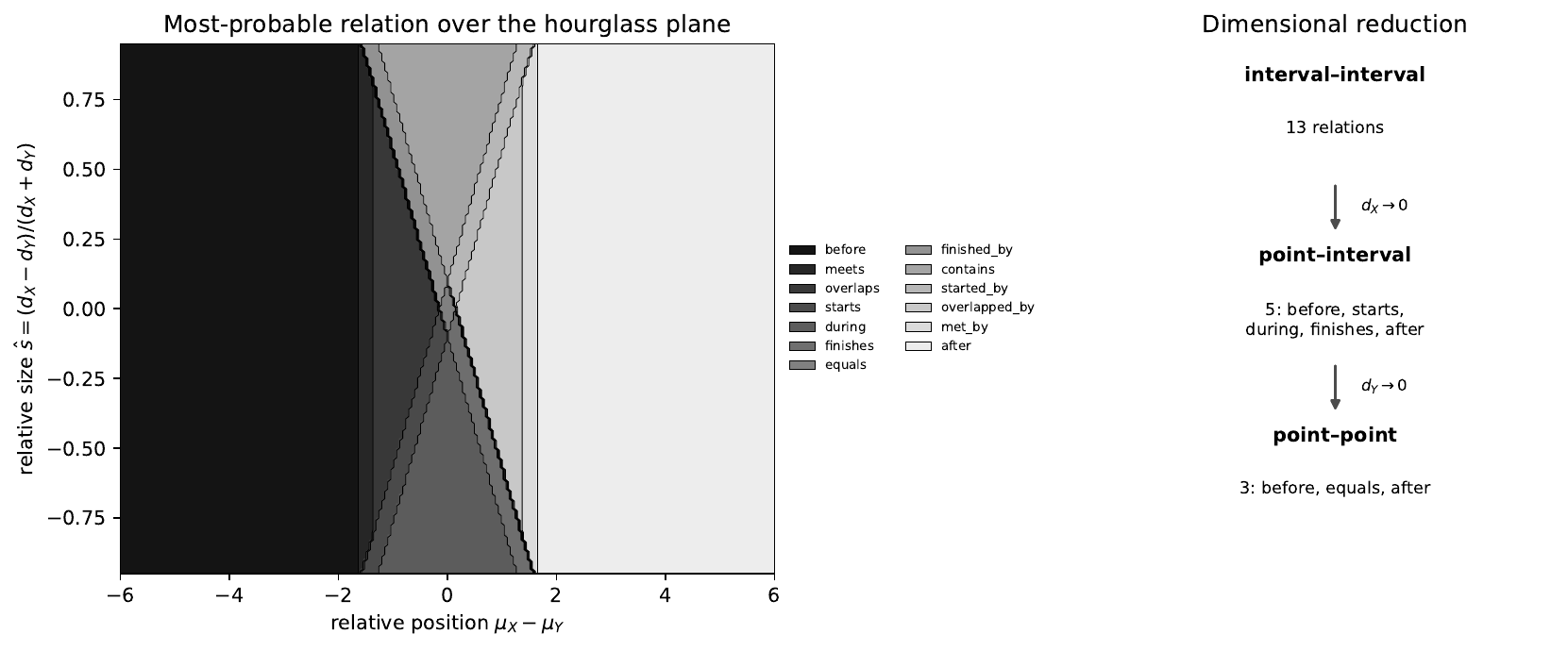}
\caption{Left: the most probable relation over the hourglass plane of relative position and relative size, reproducing the hourglass geometry from the generative model. Right: the \(13\to5\to3\) reduction as intervals collapse to points.}
\Description{Left: a two-dimensional map of relative position versus relative size, coloured by the most probable relation, reproducing the hourglass shape. Right: three panels showing the collapse from thirteen to five to three relations as intervals shrink to points.}
\label{fig:phase}
\end{figure}

The correspondence is exact in our coordinates. With \(Z=t_Y-t_X\), the start and finish comparisons satisfy
\begin{equation}
    A+B=2Z,\qquad A-B=d_X-d_Y,
\end{equation}
so the \((A,B)\) plane is the hourglass plane of relative position and relative size, rotated by \(45^\circ\). The gaps carry the duration \emph{sum}, \(G=Z-\tfrac12(d_X+d_Y)\) and \(H=-Z-\tfrac12(d_X+d_Y)\), which the two hourglass axes suppress; this third coordinate reappears as the hourglass's characteristic waist, and it is precisely what carries the contact relations \(\meets,\allowbreak\metby\) that the hourglass cannot represent. Hourglass position is thus our relative mean occurrence time and hourglass size our duration contrast, but as relative, normalized combinations rather than the per-object means.

The taxonomy is the invariant of the reduction to lower-dimensional objects. As a duration shrinks to zero the latent space loses a dimension and leaves collapse, but the coarse predicates survive:
\begin{center}
\begin{tabular}{lll}
\toprule
object pair & surviving relations & count \\
\midrule
interval--interval & all & 13 \\
point--interval & \before, \starts, \during, \finishes, \after & 5 \\
point--point & \before, \equals, \after & 3 \\
\bottomrule
\end{tabular}
\end{center}
The collapse is continuous and the partition is preserved in every limit (Figure~\ref{fig:limit}). Equivalently, the taxonomy is a hierarchical coarse-graining of the relation phase diagram, and the interval-to-point reductions conserve probability \emph{within} each taxonomic parent while redistributing it \emph{among} its leaves. This is exact.

\begin{proposition}[Taxonomy invariance under reduction]
\label{prop:reduction}
Let \(X,Y\) be uncertain intervals under the single-tolerance partition of Table~\ref{tab:prob-allen-ineq}, and let \(d_X\to0\) in distribution, so \(X\) degenerates to the Gaussian point \(t_X\). Then
\begin{enumerate}[leftmargin=*,label=(\roman*)]
    \item the thirteen-relation partition converges to the five-relation point--interval partition \(\{\before,\allowbreak\starts,\allowbreak\during,\allowbreak\finishes,\allowbreak\after\}\), and the remaining eight relations carry probability \(\to0\);
    \item every coarse predicate \(C\) that is a node of the taxonomy tree (Figure~\ref{fig:relation-taxonomy}) has \(P(C)\) continuous at the limit; and
    \item the limit reallocates probability only \emph{among the leaves of a common parent}---no mass crosses between distinct coarse nodes.
\end{enumerate}
A second collapse \(d_Y\to0\) reduces the five to \(\{\before,\allowbreak\equals,\allowbreak\after\}\) by the same statement. The taxonomy tree, not the leaf catalogue, is the invariant of the reduction.
\end{proposition}

\begin{proof}
At \(d_X=0\) the point boundaries coincide, \(a_X=b_X=t_X\), so two of the four boundary differences of Section~\ref{sec:interval-interval} become dependent: \(A=G\) and \(H=-B\), with \(B-A=d_Y\ge0\). Each of the eight relations outside \(\{\before,\allowbreak\starts,\allowbreak\during,\allowbreak\finishes,\allowbreak\after\}\) combines these identities into a contradiction under the thresholded signs---\meets\ needs \(A>\tau\) with \(|G|\le\tau\) and \(G=A\); \contains\ needs \(A>\tau,\,B<-\tau\) against \(A\le B\)---so its region is empty on \(\{d_X=0\}\) and, for \(d_X>0\), a sliver of vanishing measure; hence its probability \(\to0\), giving~(i) (the full sign bookkeeping is Appendix~\ref{app:limits}). Each surviving region converges to its point--interval cell in the coordinates \((A,B)=(a_Y-t_X,\,b_Y-t_X)\). For~(ii), a taxonomy node \(C=\bigcup_{R\in\mathcal L_C}R\) has indicator \(\mathbf 1_C\) that converges pointwise almost everywhere and is bounded by \(1\), so \(P(C)\to P(C\mid d_X{=}0)\) by dominated convergence. For~(iii), the leaf boundaries interior to a parent are the bands \(|A|\le\tau\) and \(|B|\le\tau\); as \(d_X\to0\) they shift and a vanishing leaf's mass passes to its siblings---\equals\ empties into \during, \starts, \finishes\ as the point falls inside \(Y\)---but the union defining the parent is unchanged, so no mass leaves the parent.
\end{proof}

\subsection{NLP: Relation Families, Verbal Grading, and Contact}
\label{sec:nlp-grading}

For NLP it is useful to separate three dimensions that are often conflated. The first is the qualitative relation family, for example precedence, overlap, containment, or succession. The second is the strength of evidence for that family, represented by its probability relative to alternatives at the same taxonomic level. The third is boundary contact, such as \meets, \starts, \finishes, or \equals, which requires equality margins in continuous models.

Let the top-level alternatives be the three coarse families of Section~\ref{sec:taxonomy}, each a union of leaves,
\begin{equation}
    \mathcal T=\{\mathrm{before},\mathrm{overlap},\mathrm{after}\},
\end{equation}
where \(\mathrm{before}=\{\before,\allowbreak\meets\}\), \(\mathrm{after}=\{\metby,\allowbreak\after\}\) are its converse, and \(\mathrm{overlap}\) collects the remaining nine (non-separated) leaves. Their probabilities are the corresponding leaf sums (Equation~\eqref{eq:coarse-sum}) and therefore partition the unit mass, \(p_b+p_o+p_a=1\), with
\begin{equation}
    p_b=P(\mathrm{before}),\qquad
    p_o=P(\mathrm{overlap}),\qquad
    p_a=P(\mathrm{after}).
\end{equation}
A before-type expression is licensed when
\begin{equation}
    p_b > \max(p_o,p_a),
\end{equation}
while the verbal modifier can be tied to the \emph{strength} of this dominance, which we measure by the margin
\begin{equation}
    \delta = p_b - \max(p_o,p_a)\ \in (0,1]
\end{equation}
by which the \before\ family leads its closest competitor. The adverbials at issue here are precisely the graded temporal-relation expressions of Section~\ref{sec:motivation}: phrases such as \emph{just before}, \emph{shortly before}, \emph{rather before}, \emph{some time before}, and \emph{long before} all place one event before another and differ only in how far before. They therefore need not denote different Allen leaves. They can denote different grades of the same \before\ relation family. A weak but dominant value of \(p_b\) supports \emph{shortly before}; a near-certain value of \(p_b\) supports \emph{long before}. This makes \emph{shortly before} distinct from \meets: \meets\ is a contact relation, while \emph{shortly before} is a verbal grading of precedence. Spatial language has the same structure: \emph{just above}, \emph{far to the left of}, and \emph{right beside} grade qualitative spatial relations by the same mechanism, so the construction here transfers directly to a spatial Allen-style algebra (Section~\ref{sec:spatial}).

One possible schematic mapping places three ordered thresholds \(0<\theta_1<\theta_2<\theta_3<1\) on the margin \(\delta\):
\begin{equation}
\begin{array}{rcll}
    0      < \delta \le \theta_1 &\Rightarrow& \text{``just / shortly before''} & (\text{weakly dominant}),\\
    \theta_1 < \delta \le \theta_2 &\Rightarrow& \text{``rather before''}        & (\text{moderately dominant}),\\
    \theta_2 < \delta \le \theta_3 &\Rightarrow& \text{``some time before''}     & (\text{strongly dominant}),\\
    \theta_3 < \delta \le 1        &\Rightarrow& \text{``long before''}          & (\text{near-total dominance}).
\end{array}
\label{eq:adverbial-cuts}
\end{equation}
Here ``moderately'' and ``strongly dominant'' are simply the middle bands \(\theta_1<\delta\le\theta_2\) and \(\theta_2<\delta\le\theta_3\); a representative calibration is \((\theta_1,\theta_2,\theta_3)=(0.2,0.5,0.8)\), and \(\delta\to1\) exactly as \(p_b\to1\).
The thresholds are not universal lexical constants. They may depend on language, genre, and task. The important point is that the algebra supplies a normalized graded variable that can be calibrated empirically.

Adjacency-dependent relations behave differently. For \meets, \metby, \starts, \startedby, \finishes, \finishedby, and \equals, the relevant probability depends directly on a tolerance or soft boundary kernel:
\begin{equation}
    P(\meets_{\tau})=P(|a_Y-b_X|\le \tau,\ a_X<a_Y,\ b_X<b_Y),
\end{equation}
\begin{equation}
    P(\starts_{\tau})=P(|a_Y-a_X|\le \tau,\ b_X<b_Y).
\end{equation}
Thus contact words such as \emph{touches}, \emph{starts with}, \emph{ends with}, or \emph{at the same time} are margin-sensitive, whereas non-contact verbal gradings are primarily relation-probability-sensitive. The same framework supports both, but it keeps their semantics separate.

This distinction gives a bidirectional NLP interface. In interpretation, linguistic evidence can be translated into constraints or likelihood terms over relation probabilities and boundary margins. In generation, a system can choose the relation family by the most probable taxonomic predicate and choose the verbal modifier by the strength of that probability, while reserving contact expressions for high margin-dependent contact probability.

A fourth cue is structural rather than gradable: an utterance usually reveals \emph{which} temporal quantities the speaker has fixed, and hence how its interval should be parameterized (Section~\ref{sec:parameterization}). ``I arrived on the 11th and stayed about two months'' anchors a start and leaves the end uncertain; ``ten days during the summer'' fixes a duration and leaves its position uncertain; ``a week's holiday, home by Wednesday evening'' anchors an end. The same sentence therefore supplies both the relation-level evidence above and the boundary-level uncertainty beneath it---which anchor is sharp, which quantity floats---and the framework consumes both through the same Gaussian parameters.

\subsection{From Discrete Cuts to Graded Membership}
\label{sec:graded-membership}

The mapping~\eqref{eq:adverbial-cuts} is a \emph{hard partition}---each \(\delta\) yields one adverbial---which human usage overruns: asked which word fits a configuration, speakers spread across adjacent adverbials (\emph{just} and \emph{recently} both fit a recent event, the balance tipping smoothly as it recedes). This is a graded variability a hard partition cannot express, since fit to the measured responses it collapses each configuration onto a single winner. The remedy, confirmed on the human data of the companion paper~\cite{paa2026nlp}, is a \emph{soft} read-out: one smooth membership kernel per adverbial over the margin, \(A_a(\delta)\in[0,1]\), overlapping its neighbours, so that a configuration yields a distribution \(P(a\mid\delta)\propto A_a(\delta)\) over the competing adverbials, with~\eqref{eq:adverbial-cuts} recovered only as the kernel width vanishes---smooth bells on the (log-)margin, a skew-normal per adverbial, reproduce the measured usage closely where the cuts fail. An embodied or conversational agent closes the loop from this distribution, emitting a \emph{single} word either as the mode \(\operatorname{argmax}_a P(a\mid\delta)\), the canonical choice, or by sampling \(a\sim P(a\mid\delta)\) to reproduce natural variability; and the same kernels run backward, a heard adverbial contributing the likelihood \(A_a(\delta)\) for the inverse map to concrete times.

This soft read-out is not merely a repair of the hard cuts: it can be estimated directly from how people speak, and a companion paper~\cite{paa2026nlp} confirms it on a human corpus of graded recency judgements spanning the before-family adverbials \emph{just}, \emph{recently}, \emph{some time ago}, and \emph{long time ago}. Each adverbial is given one smooth membership kernel over the margin---a \emph{skew-normal}, a Gaussian allowed to fall off at different rates on its two sides, since adverbial preference is asymmetric (\emph{just} drops sharply toward the very recent yet fades slowly as an event recedes toward \emph{recently}). Fit to the data these kernels reproduce the measured usage closely, picking the expressed word at a macro-\(F_1\) of \(0.84\), where the one-hot cuts of Equation~\eqref{eq:adverbial-cuts} cannot express the graded spread at all. With soft margins we can thus model the variability of adverbial expression accurately, and a human corpus of graded temporal adverbials confirms it. Figure~\ref{fig:adverbial-usage} is one illustrative event: as the distance grows the expressed word walks \emph{just}\(\to\)\emph{recently}\(\to\)\emph{some}\(\to\)\emph{long}, and the fitted read-out tracks that walk. The companion paper~\cite{paa2026nlp} develops the estimation, metrics, and kernel comparison in full.

\begin{figure}[t]
\centering
\includegraphics[width=\linewidth]{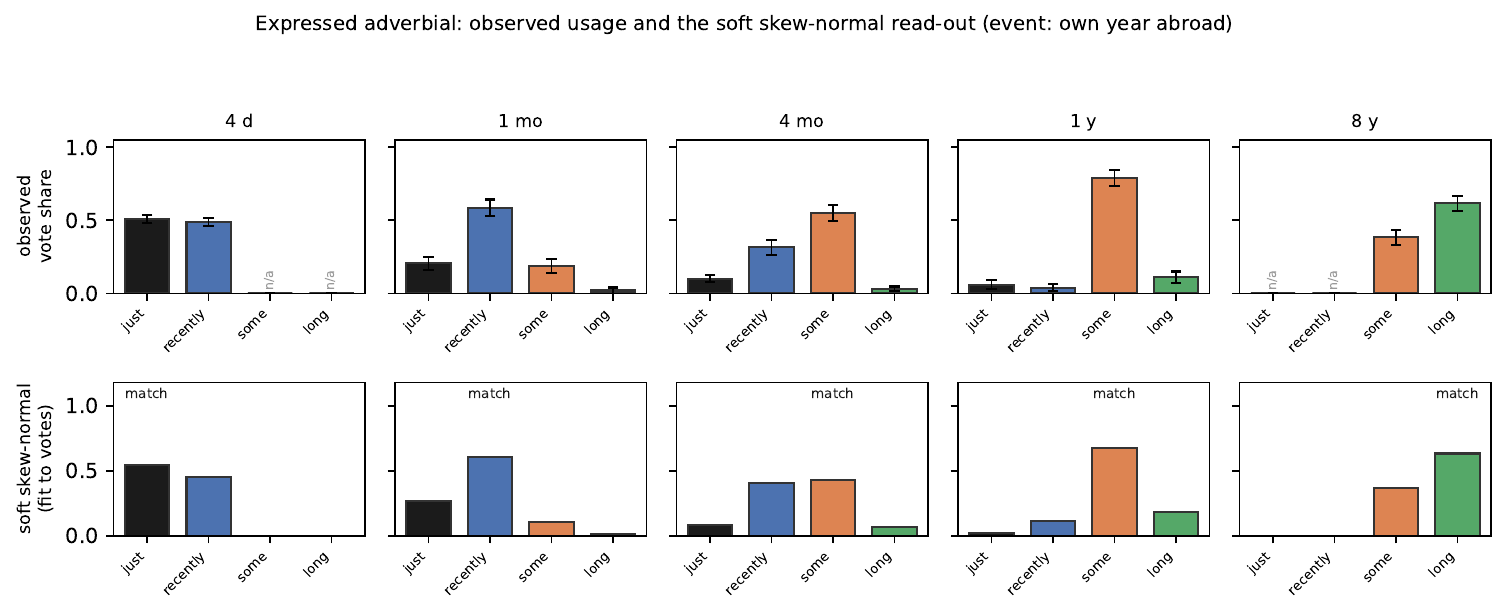}
\caption{Graded recency adverbials that persons use to describe the recency of something that happens after a single event---e.g.~a person asked at a given time period after a \emph{year abroad}---as the time period recedes from a few days into the past (left) to several years ago (right). \emph{Top:} which word people actually reach for, shown as the share of speakers preferring each one; the preference slides smoothly from \emph{just} through \emph{recently} and \emph{some time ago} to \emph{long time ago} as the event ages (gaps mark distances at which a word was not offered). \emph{Bottom:} the soft read-out---one skew-normal membership curve per word over the beforeness margin---reproducing that sliding preference, its mode matching the observed word at every distance. An illustrative example; the companion paper~\cite{paa2026nlp} reports the full data and quantitative evaluation.}
\Description{Stacked curves over elapsed time showing which recency adverbial, from just now to long ago, people use for an event as it recedes into the past, with a smooth hand-over from one adverbial to the next.}
\label{fig:adverbial-usage}
\end{figure}

\subsection{Scale Invariance and Context-Sensitivity Normalization}
\label{sec:scale-invariance}

A desirable property of the distribution-induced formulation is scale invariance. Suppose all metric temporal quantities are scaled by a positive constant \(c>0\):
\begin{equation}
    t' = c t,
    \qquad
    d' = c d,
    \qquad
    \mu' = c\mu,
    \qquad
    \sigma' = c\sigma,
    \qquad
    \tau' = c\tau.
\end{equation}
Then all standardized Gaussian arguments used in the relation probabilities remain unchanged. For the point--point before relation,
\begin{equation}
    P(x<y)=
    \PhiCDF\left(
        \frac{\mu_y-\mu_x}{\sqrt{\sigma_x^2+\sigma_y^2}}
    \right),
\end{equation}
and after scaling,
\begin{equation}
    \frac{c(\mu_y-\mu_x)}{\sqrt{c^2\sigma_x^2+c^2\sigma_y^2}}
    =
    \frac{\mu_y-\mu_x}{\sqrt{\sigma_x^2+\sigma_y^2}}.
\end{equation}
The same cancellation holds for interval relations. For example, the thresholded separation term
\begin{equation}
    \beta_+(\tau)=
    \frac{
        (\mu_Y-\mu_X)-\tau-\frac12(\mu_{d_X}+\mu_{d_Y})
    }{
        \sqrt{\sigma_X^2+\sigma_Y^2+\frac14(\sigma_{d_X}^2+\sigma_{d_Y}^2)}
    }
\end{equation}
is invariant under simultaneous scaling of means, durations, standard deviations, and thresholds. The correlation matrices are invariant as well, because their entries are ratios of standard deviations.

This provides a principled normalization of context-sensitivity. A one-month gap before a two-month interval and a one-hour gap before a two-hour interval should receive the same graded beforeness if the uncertainty and duration parameters scale in the same way. The model therefore does not require a fixed absolute meaning for \emph{shortly} or \emph{long}. Instead, verbal grading can be grounded in dimensionless relation probabilities, while domain- or discourse-specific scales enter through the distributions and tolerances themselves.

The property is not meant to remove all context. Language still depends on event type, discourse goals, and pragmatic expectations. Rather, the algebra separates two sources of context-sensitivity. Metric scale is normalized by standardized probability calculations; lexical and pragmatic variation can then be learned as mappings from relation probabilities, expected gaps, and task context to verbal expressions.

\subsection{Where the Reference Scale Comes From: Extents, Spreads, and the Point Limit}
\label{sec:reference-scale}

Scale invariance (Section~\ref{sec:scale-invariance}) holds because every argument of a relation probability is a standardized ratio. Grading a relation into verbal degrees---\emph{shortly} versus \emph{long before}, \emph{well inside} versus \emph{dead centre}---likewise requires a dimensionless \emph{margin}, and it is worth making explicit \emph{which} quantity supplies the yardstick, because the answer differs across relation families and changes under the interval\(\to\)point limits of Appendix~\ref{app:limits}. Two yardsticks arise.

\paragraph{Separation is measured against the spread.} For the precedence and succession families the graded quantity is the \emph{gap} between the two objects, standardized by the combined uncertainty. In the point--interval case (Section~\ref{sec:point-interval}) the standardized separation at zero margin is
\begin{equation}
    \beta=\frac{\Delta_{Yx}-\tfrac12\mu_{d_Y}}{\sigma_{pI}}
        =\frac{\text{expected gap}}{\text{combined spread}},
    \label{eq:sep-margin}
\end{equation}
and the dominance margin \(\delta\) of Section~\ref{sec:nlp-grading} reads this ratio through a saturating link (\(\delta=2\PhiCDF(\beta)-1\) for two points). The yardstick is the spread \(\sigma\); no interval extent is involved.

\paragraph{Containment is measured against an extent.} For the containment and overlap families the natural graded quantities are ratios of a boundary offset to a \emph{duration}:
\begin{equation}
    \kappa=\frac{d_X}{d_Y}\in[0,1]\quad(\text{coverage}),
    \qquad
    \gamma=\frac{2\,(t_X-t_Y)}{d_Y}\in[-1,1]\quad(\text{centrality}),
    \label{eq:contain-margins}
\end{equation}
together with the insideness margins \((a_X-a_Y)/d_Y\) and \((b_Y-b_X)/d_Y\). These are ratios of times, hence dimensionless, and---unlike the separation gap---\emph{bounded}: \emph{throughout} demands \(\kappa\to1\), \emph{in the middle of} demands \(\gamma\to0\), \emph{well inside} demands large insideness margins. The yardstick is the container's own extent \(d_Y\), already present in the configuration.

\paragraph{The point limit selects the yardstick.} Which yardstick survives is decided by what degenerates. Sending the \emph{contained} object to a point (\(d_X\to0\)) sends coverage \(\kappa\to0\)---a point covers nothing---so the \emph{throughout} family collapses. But centrality \(\gamma\) and the insideness margins depend on the \emph{container}'s extent \(d_Y\), not the point's, and survive intact: \emph{a point inside an interval still has a well-defined, bounded, scale-free position within it}, because the interval supplies the yardstick. Sending \emph{both} objects to points removes every extent at once, leaving only the spread \(\sigma\).

\paragraph{The principle, and the one configuration that is not self-scaling.} A graded relation carries an intrinsic margin normalized by a quantity the configuration already contains: an interval \emph{extent} for the containment and position gradings, present whenever the container is not a point; the joint \emph{spread} for the separation gradings, present always. The separation margin \(\delta\) is thus always scale-free---but it \emph{saturates} (\(\delta\to1\) for well-separated objects), so resolving fine degrees in the far tail---\emph{long} versus \emph{very long before}---needs a \emph{metric} distance scale beyond the probability. When an interval extent survives it supplies one intrinsically: a longer event sets a longer scale, as expected---\emph{shortly after} a decade-long war spans years, \emph{shortly after} a brief meeting only minutes. Between two \emph{bare points}---a past event and the utterance instant---no extent remains to play this role, and the metric scale must come from outside the geometry, from the event \emph{type}, and enters on a logarithmic (ratio) axis, in keeping with the Weber--Fechner character of human time perception~\cite{fechner1860elemente,zauberman2009discounting}. This is the single configuration in which grading is not intrinsically scale-free---the fully collapsed, saturated separation of two points---and it is exactly where a factorized, event-type-conditioned account of temporal distance must locate a per-event characteristic timescale \cite{kenneweg2025factorized}: the scale a point event lacks by geometry, restored by its kind. Everywhere an interval extent survives, no external scale is needed.

\paragraph{Why that imported scale is logarithmic---and why the sigmoid is not already it.} The logarithm acts only on this residual magnitude, never on the relation itself. Which relation holds is fixed by the \emph{signs} of the boundary differences, and the orthant probabilities that decide it require those differences to be Gaussian in \emph{linear} time---a logarithm is not even defined across a sign change. Only once a relation is settled does a positive magnitude remain to be graded, and the scale on which it is read is then free. One might expect the saturating link to supply the compression already: with \(p=\PhiCDF(\Delta/\sigma)\) the before-probability, \(\delta=2p-1\) does flatten large separations. But it flattens with a Gaussian \emph{tail}, not logarithmically. Measuring resolution as the change in \(p\) per \emph{ratio} of distance, \(\mathrm{d}p/\mathrm{d}\ln\Delta=(\Delta/\sigma)\,\phi(\Delta/\sigma)\), the linear-time law vanishes \emph{super-exponentially} in the far tail: configurations of very different magnitude---\emph{long} versus \emph{very long before}---collapse onto \(p\approx1\) and become mutually indistinguishable. A logarithmic argument, \(p=\PhiCDF(\ln(\Delta/t_0)/s)\) (with the dimensionless log-width \(s\)), instead spreads resolution \(\tfrac1s\phi(\ln(\Delta/t_0)/s)\) \emph{evenly across decades}. The saturating sigmoid is thus \emph{a} compression but the wrong one; scale-free grading of the far tail requires the logarithm---precisely the external ratio scale the collapsed point pair must import, and the one on which human elapsed-time judgement runs \cite{fechner1860elemente,zauberman2009discounting}.

\subsection{From Geometry to a Testable Hypothesis}
\label{sec:scale-hypothesis}

The yardsticks of Section~\ref{sec:reference-scale} are \emph{derived}: the algebra fixes which quantity normalizes each family's margin. Whether human grading actually uses them is an empirical question, and because the geometry is sharp so are the predictions. We state them as a falsifiable hypothesis---identifying where the two extents combine, where an extent is present yet contributes nothing, and where a further, non-geometric factor must enter.

\paragraph{Two extents, and the sign the relation fixes.} When both objects are extended, both extents feed the scale, and the relation family decides how. A \emph{separation} grades the gap, which carries the duration \emph{sum} \(G=Z-\tfrac12(d_X+d_Y)\): lengthening \emph{either} object closes the gap alike, so the two extents enter with the \emph{same} sign. A \emph{containment} is graded against the container's extent alone, through the coverage \(\kappa=d_X/d_Y\): enlarging the container \(d_Y\) supports the relation while enlarging the contained \(d_X\) opposes it, so the extents enter with \emph{opposite} signs. The prediction is a measurable sign structure---summed extents for separation, container-minus-contained for containment---keyed to the relation and visible wherever the two extents are varied independently.

\paragraph{Present but inert: the pinning of the boundary.} An extent contributes only when the relation turns on the boundary that extent \emph{moves}, and which boundary that is depends on how the object is pinned (Section~\ref{sec:parameterization}). A point behind an \emph{end}-pinned interval grades on the gap to the pinned end, and the interval's duration---extending backward from that end---never enters it; a point before a \emph{start}-pinned interval is the mirror case. The extent becomes active only when the relation turns on the \emph{unpinned}, duration-carried edge, as for a point after a start-pinned interval. The prediction---duration inert in the pinned-edge configuration, active in the unpinned one---is borne out by companion analysis of the FuzzyLLI corpus \cite{kenneweg2025factorized}: elapsed-time judgements of past events pin the recent edge, and the event's duration leaves graded pastness unchanged, whereas re-pinning the onset makes the same duration matter. Pinning these parameters to their known physical values rather than fitting them is enough: on the spatial companion corpus, fixing the boundary to the given geometry and freeing only the positional uncertainty still recovers the human proximity judgements (\(R^2\approx0.97\), matching the free fit), so the edge-shift prediction holds in the operational regime an agent actually faces~\cite{paa2026nlp}.

\paragraph{Where geometry supplies no scale: a role for salience.} Between two \emph{bare points} no extent survives (Section~\ref{sec:reference-scale}) and the scale must be imported from the event \emph{kind}, on a logarithmic axis. We add a tentative refinement: what the kind supplies is better read as the event's \emph{salience} than its physical extent. In the same companion analysis the fitted per-event timescale is predicted by rated importance, frequency, and richness rather than by duration---a brief but consequential event stays ``recent'' as long as a protracted one. An extent may therefore \emph{feed} the imported scale (longer events are often more salient) without \emph{setting} it. The working hypothesis is thus that the grading scale is a \emph{relation-dependent combination of the surviving extents, modulated by salience}, collapsing to a purely salience- and kind-set scale in the fully degenerate point pair.

\paragraph{A clean test.} Separating these claims needs objects of independently varied extent in both separation and containment relations---two disks of variable radius in a spatial judgement, say, where \emph{near} should grade the gap against the \emph{summed} radii while \emph{inside} or \emph{around} should grade against the \emph{container} radius alone, and where matching salience across sizes isolates the geometric contribution. The algebra fixes the geometric skeleton of these predictions; behavioural data must supply the salience weighting and settle the degenerate cases.

\subsection{Robustness and Inadequacy of Flat Relation Assignment}

The inadequacy of flat relation assignment is not merely numerical. It is structural. If a model predicts thirteen independent relation scores, it does not know that \meets\ is a specialization of precedence/contact, or that \starts\ is a specialization of containment. It may therefore report low confidence in every fine relation while missing that a coarse relation has high probability.

The distribution-induced taxonomy avoids this. Coarse probabilities are computed as sums over leaves or directly as broader boundary inequalities. This lets reasoning proceed at the appropriate level of granularity: first infer that \(X\) is probably contained in \(Y\), then refine into \starts, \during, \finishes, or \equals; first infer that \(X\) precedes or touches \(Y\), then refine into strict \before\ or approximate \meets. This hierarchical behavior is especially important near boundary cases, where classical Allen leaves change discontinuously but the parent predicate remains stable.

\subsection{Selecting a Best-Fitting Relation}
\label{sec:best-relation}

A natural request is a single best-fitting relation for a configuration. The obvious choice, the maximiser of the thirteen leaf probabilities, is biased by relation measure and should be used with care. The six relations defined by strict inequalities (\before, \overlaps, \during, \contains, \overlappedby, \after) are full-dimensional regions, whereas the seven equality and contact relations occupy thin confidence bands of width on the order of \(\tau\). The probability of \meets\ or \starts\ is therefore structurally small, and a flat \(\argmax\) over all thirteen leaves systematically disfavours contact relations and depends on \(\tau\): as \(\tau\to0\) a contact relation can never win. Comparing \(P(\before)\) directly with \(P(\meets)\) compares the mass of a region with that of a boundary. A flat maximum is thus meaningful only among siblings of comparable measure, for example within a single taxonomic level such as precedence versus overlap versus succession.

Two selection rules respect the structure. The first is a top-down maximum-a-posteriori decoding over the taxonomy: choose the most probable coarse family by its summed leaf mass, descend into it, and refine, so that every decision weighs alternatives of comparable measure. Because the coarse probabilities are large and stable, the family decision is robust even where the leaf decision is delicate. The second is the relation of the most probable arrangement: classify the mode of the latent distribution, that is, the modal midpoints and durations, which yields a definite configuration whose relation is read off without any band-width bias.

These two rules answer different questions. The taxonomic maximum identifies the most probable relation \emph{category} under uncertainty, while the modal-arrangement relation identifies the relation of the single most likely \emph{configuration}. They agree for well-separated full-dimensional relations and can diverge near contact: when the modal boundaries align within \(\tau\), the modal-arrangement relation is a contact relation such as \meets, while the probability mass may still favour a thicker neighbour such as \before\ or \overlaps. Both are legitimate; the taxonomy makes explicit that the robust quantities are the coarse predicates, and that fine selection should proceed conditionally within them rather than as a flat comparison across heterogeneous leaves. The accompanying implementation provides all three selectors.

\subsection{Extension to Spatial Relations}
\label{sec:spatial}

The construction transfers from time to space without new probabilistic machinery. Allen's algebra already has a well-developed qualitative spatial counterpart: applying the interval relations independently to each Cartesian axis yields the \emph{rectangle algebra} for axis-aligned rectangles, with \(13^2=169\) base relations, and in \(n\) dimensions the \emph{block algebra} with \(13^n\) relations \cite{gusgen1989spatial,balbiani1998rectangle,balbiani1999block}. Two further calculi describe orthogonal facets of the same scene: the region connection calculus represents topology---disconnected, touching, overlapping, contained \cite{randell1992rcc}---and the cardinal direction calculus represents projective direction \cite{skiadopoulos2004cardinal}. These spatial relations are the analogue of the adverbial families of Section~\ref{sec:nlp-grading}: \emph{just above}, \emph{far to the left of}, and \emph{right beside} grade a directional relation exactly as \emph{shortly before} grades precedence.

Our generative semantics lifts to the rectangle algebra directly. A rectangle is two intervals, one per axis, each carrying the same midpoint-and-duration uncertainty used here; every edge is therefore an affine function of the latent Gaussians, and a rectangle relation \((r_x,r_y)\) holds on the intersection of the two axes' linear-inequality regions. Its probability is the product of the per-axis orthant integrals when the axes are independent, and a single higher-dimensional Gaussian orthant integral when they are correlated---the three- and four-dimensional integrals of the temporal case become six- and eight-dimensional. The \(13^n\) relations remain a complete partition for any tolerance, and contact relations again acquire positive measure through the tolerance band. The covariance-agnostic engine of Appendix~\ref{app:parameterizations} already evaluates the correlated case unchanged.

What is novel is the combination, not the ingredients. The qualitative rectangle and block algebras are entirely crisp \cite{balbiani1998rectangle,balbiani1999block}. Uncertainty has reached spatial relations only as fuzzy membership---Bloch's morphological ``fuzzy landscape'' for directional relations \cite{bloch1999fuzzy}, and fuzzy two-dimensional Allen relations \cite{salamat2012fuzzy}---or as probabilities \emph{annotated} onto qualitative relations and propagated through a network, as in probabilistic region connection calculus \cite{girlea2015prcc}. The one line of work that, like ours, induces relation probabilities from a distribution over geometry assigns them to topological relations through a raster error model rather than in closed form \cite{winter2000uncertain}; and the computer-vision practice of modelling box corners as Gaussians \cite{hall2020probabilistic} supplies our affine-Gaussian premise but never connects it to a relation algebra. Learned generative models of spatial language likewise fit a separate density per relation, without a closed-form partition \cite{dawson2013generative}. A distribution-induced spatial algebra---rectangle and block relations as closed-form Gaussian orthant probabilities over a complete partition---therefore appears to be unoccupied, and is the natural next step for this work.

\subsection{Limitations and Future Work}

Several modelling choices can be relaxed. Midpoints and durations are taken to be independent and the primitive Gaussian variables uncorrelated; correlated temporal quantities require only the appropriate joint covariance in place of the diagonal one. Computation rests on multivariate Gaussian CDFs, which lack an elementary closed form beyond the bivariate case and are evaluated by specialised quadrature; this is inexpensive for the tri- and quadri-variate systems used here but is worth noting for large temporal networks.

Three extensions are natural. First, because the relation probabilities are smooth functions of the temporal parameters, their gradients provide a differentiable signal of temporal tendency, connecting to robustness gradients in signal temporal logic \cite{fainekos2009robustness,donze2010robust,haghighi2019smooth,leung2020backprop}. Second, the constraint-compilation view supports posterior updating from observed relations and network-level inference over many events, in the spirit of probabilistic interval networks \cite{ryabov2004probabilistic}. Third, the taxonomy and scale invariance suggest a calibrated mapping between relation probabilities and vague temporal language, building on factorized adverbial models \cite{kenneweg2025factorized} and fuzzy interval algebras \cite{schockaert2008fuzzifying,badaloni2006iafuz}; benchmark evidence that language models resolve vague temporal references markedly worse than explicit ones \cite{kenneweg2025traveler} underscores the need for such a principled mapping.

\subsection{Conclusion}

We have presented a probabilistic Allen algebra in which uncertain temporal points and intervals induce probabilities over qualitative relations. The approach unifies point--point, point--interval, and interval--interval reasoning; gives closed forms using error functions and multivariate Gaussian CDFs; treats equality/contact through confidence borders; and expresses relation taxonomy directly in terms of calculated probabilities. For NLP, the framework distinguishes contact from verbal grading and provides scale-invariant relation probabilities that normalize much of the context-sensitivity of vague temporal adverbials. This positions the framework as a bridge between metric temporal uncertainty, qualitative interval reasoning, temporal-relation taxonomies, and vague temporal language.

\appendix

\section{The Thirteen Relations, Grouped}
\label{app:relation-chart}

Figure~\ref{fig:relation-chart} lays out the same thirteen relations as the taxonomy tree of Figure~\ref{fig:relation-taxonomy}, but as nested boxes: the two-way split into separated and non-separated configurations boxes the coarse families, and each family boxes its Allen leaves. Nested boxes are nested partition cells, so a family's probability is the sum of the probabilities of the sketches it contains.

\begin{figure}[htbp]
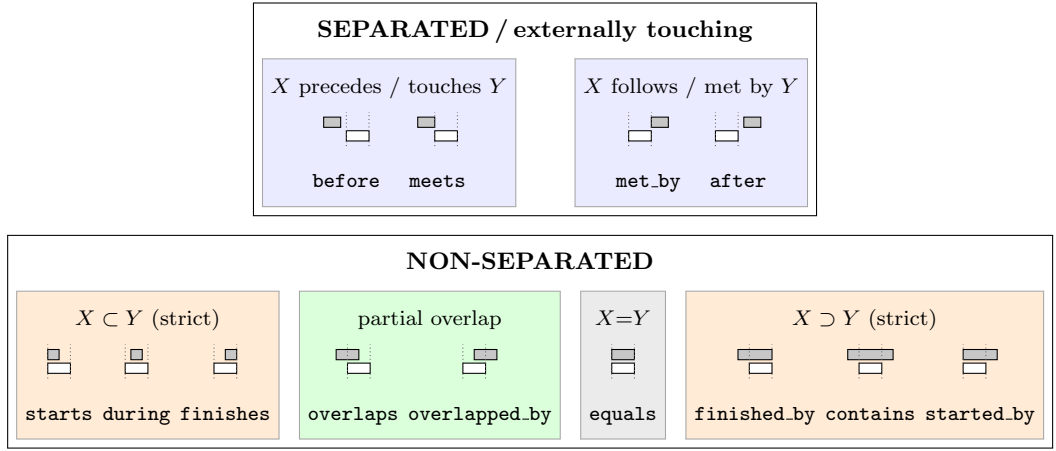

\centering
\fbox{\begin{tabular}{@{}c@{}}
{\footnotesize\bfseries SEPARATED\,/\,externally touching}\\[4pt]
\taxfam{blue!8}{$X$ precedes / touches $Y$}{\begin{tabular}{@{}c@{\quad}c@{}}\taxcellf{0}{1.5}{\before} & \taxcellf{0.5}{2}{\meets}\end{tabular}}\qquad
\taxfam{blue!8}{$X$ follows / met by $Y$}{\begin{tabular}{@{}c@{\quad}c@{}}\taxcellf{4}{5.5}{\metby} & \taxcellf{4.5}{6}{\after}\end{tabular}}
\end{tabular}}

\vspace{6pt}

\fbox{\begin{tabular}{@{}c@{}}
{\footnotesize\bfseries NON-SEPARATED}\\[4pt]
\begin{tabular}{@{}c@{\ \ }c@{\ \ }c@{\ \ }c@{}}
\taxfam{orange!16}{$X\subset Y$ (strict)}{\begin{tabular}{@{}c@{\ }c@{\ }c@{}}\taxcellf{2}{3}{\starts} & \taxcellf{2.5}{3.5}{\during} & \taxcellf{3}{4}{\finishes}\end{tabular}}
&
\taxfam{green!14}{partial overlap}{\begin{tabular}{@{}c@{\ }c@{}}\taxcellf{1}{3}{\overlaps} & \taxcellf{3}{5}{\overlappedby}\end{tabular}}
&
\taxfam{black!8}{$X{=}Y$}{\taxcellf{2}{4}{\equals}}
&
\taxfam{orange!16}{$X\supset Y$ (strict)}{\begin{tabular}{@{}c@{\ }c@{\ }c@{}}\taxcellf{1}{4}{\finishedby} & \taxcellf{1}{5}{\contains} & \taxcellf{2}{5}{\startedby}\end{tabular}}
\end{tabular}
\end{tabular}}
\caption{The thirteen relations grouped by their partition. In every cell the reference interval $Y$ (outlined) is held fixed---dotted ticks mark its edges---and the event $X$ (filled) moves against it. Colour marks the coarse family, converse families sharing a hue and mirrored left$\leftrightarrow$right (\before/\after, \starts/\startedby, \dots). Nested boxes mirror the taxonomy of Figure~\ref{fig:relation-taxonomy}: a family's probability is the sum of its members'---e.g.\ $P(X\subseteq Y)=P(\starts)+P(\during)+P(\finishes)+P(\equals)$. The two containment blocks are the \emph{strict} directions; \equals\ is the separate central cell they share as a boundary.}
\Description{A chart of the thirteen relations grouped into coloured coarse-predicate boxes; each cell sketches the event X against a fixed reference interval Y with dotted ticks at the edges of Y.}
\label{fig:relation-chart}
\end{figure}

\section{Limiting Cases When Intervals Collapse to Points}
\label{app:limits}

Degenerate durations should be handled as limits. If \(d_X=d_Y=0\), then \(X=\{t_X\}\) and \(Y=\{t_Y\}\). The only non-zero exact relations are
\begin{equation}
    P(X\ \before\ Y)=\PhiCDF\left(\frac{\Delta}{\sigma_Z}\right),
    \qquad
    P(X\ \after\ Y)=\PhiCDF\left(\frac{-\Delta}{\sigma_Z}\right),
\end{equation}
where \(\Delta=\mu_Y-\mu_X\) and \(\sigma_Z^2=\sigma_X^2+\sigma_Y^2\). Approximate equality is
\begin{equation}
    \Prob(|t_Y-t_X|\le\tau)
    =\PhiCDF\left(\frac{\tau-\Delta}{\sigma_Z}\right)
     -\PhiCDF\left(\frac{-\tau-\Delta}{\sigma_Z}\right).
\end{equation}
Exact equality has probability zero unless both points are deterministically identical.

If only \(X\) collapses to a point, then \(X\subseteq Y\) becomes point-in-interval:
\begin{equation}
    P(X\subseteq Y)=\Prob\left(|t_Y-t_X|\le\frac{d_Y}{2}\right),
\end{equation}
while \(Y\subseteq X\) is zero unless \(Y\) also collapses to the same point. Boundary correspondence becomes approximate equality between the point and \(a_Y\) or \(b_Y\).

\paragraph{The thirteen-to-five fold in detail.}
This is the sign bookkeeping behind Proposition~\ref{prop:reduction}. Setting \(d_X=0\) in the midpoint parameterization gives \(a_X=b_X=t_X\), and the four boundary differences of Section~\ref{sec:interval-interval} collapse onto two through the identities
\begin{equation}
    A=a_Y-t_X=G,\qquad H=t_X-b_Y=-B,
\end{equation}
leaving the independent pair \((A,B)=(a_Y-t_X,\,b_Y-t_X)\) with \(B-A=d_Y\ge0\), so \(A\le B\). Reading Table~\ref{tab:prob-allen-ineq} under these identities sorts the thirteen relations exactly. The five survivors reduce to point position within \(Y\):
\begin{center}
\begin{tabular}{lll}
\toprule
relation & inequality at \(d_X=0\) & point \(t_X\) relative to \(Y\) \\
\midrule
\before   & \(A>\tau\)               & before \(a_Y\) \\
\starts   & \(|A|\le\tau,\ B>\tau\)  & at \(a_Y\) \\
\during   & \(A<-\tau,\ B>\tau\)     & strictly inside \\
\finishes & \(A<-\tau,\ |B|\le\tau\) & at \(b_Y\) \\
\after    & \(B<-\tau\)              & after \(b_Y\) \\
\bottomrule
\end{tabular}
\end{center}
The other eight are infeasible under \(A=G\), \(H=-B\), and \(A\le B\): \meets\ (\(A>\tau,\,|G|\le\tau\)) and \overlaps\ (\(A>\tau,\,G<-\tau\)) both set \(G=A>\tau\); \finishedby\ (\(A>\tau,\,|B|\le\tau\)) and \contains\ (\(A>\tau,\,B<-\tau\)) violate \(A\le B\); \startedby\ (\(|A|\le\tau,\,B<-\tau\)) forces \(A\le B<-\tau\); \overlappedby\ (\(B<-\tau,\,H<-\tau\)) and \metby\ (\(B<-\tau,\,|H|\le\tau\)) both set \(H=-B>\tau\); and \equals\ (\(|A|\le\tau,\,|B|\le\tau\)) requires \(d_Y=B-A\le2\tau\). Seven are thus outright empty on \(\{d_X=0\}\); the eighth, \equals, is feasible only for a near-degenerate \(Y\) (\(d_Y\le2\tau\)) and otherwise contributes vanishing probability, its mass passing to \during, \starts, and \finishes\ as the collapsing \(X\) falls inside \(Y\). For \(d_X>0\) each of the eight occupies a band of width \(O(d_X)\) that closes as \(d_X\to0\). Because every coarse taxonomy node is a fixed union of leaves, dominated convergence carries its probability continuously through the limit; the reduction therefore conserves probability within each parent while reallocating it among the leaves above. The symmetric collapse \(d_Y\to0\) eliminates \(B\) against \(A\) the same way and reduces the five point--interval relations to \(\{\before,\allowbreak\equals,\allowbreak\after\}\); applying both collapses recovers the two-point case at the head of this appendix.

\section{Alternative Interval Parameterizations}
\label{app:parameterizations}

The midpoint--duration parameterization used in the body is convenient but not required. An interval has two degrees of freedom, and any pairing of one location quantity with the duration fixes it: midpoint and duration \((t,d)\), start and duration \((s,d)\), or end and duration \((e,d)\), related by \(s=t-d/2\) and \(e=t+d/2\)---or the two boundaries themselves, start and end \((s,e)\), with the duration as their difference \(d=e-s\). The duration \(d\ge0\) is common to the first three, whose location quantities differ only by \(\pm d/2\); in the fourth it is derived, and \(d\ge0\) reads as the end not preceding the start. In each case both boundaries are affine in the latent variables, so every Allen relation stays a conjunction of linear inequalities and its probability a conditional multivariate Gaussian orthant probability.

Concretely, the thirteen relations see the two intervals only through the four boundary differences
\begin{equation}
    A=a_Y-a_X,\qquad B=b_Y-b_X,\qquad G=a_Y-b_X,\qquad H=a_X-b_Y
\end{equation}
(the quantities tabulated in Table~\ref{tab:prob-allen-ineq}), so a parameterization is pinned down once these are written in its latent variables. Throughout, \(d_X,d_Y\ge0\) denote the truncated-Gaussian durations, shared by the three location--duration forms.

\paragraph{Start and duration.} Take \(X=[s_X,\,s_X+d_X]\), so \(a_X=s_X\) and \(b_X=s_X+d_X\). With the start difference \(S=s_Y-s_X\sim\N(\mu_{s_Y}-\mu_{s_X},\,\sigma_{s_X}^2+\sigma_{s_Y}^2)\),
\begin{equation}
    A=S,\qquad B=S+d_Y-d_X,\qquad G=S-d_X,\qquad H=-S-d_Y.
\end{equation}
Here the latent start difference \emph{is} the start/start primitive \(A\).

\paragraph{End and duration.} Take \(X=[e_X-d_X,\,e_X]\), so \(a_X=e_X-d_X\) and \(b_X=e_X\). With the end difference \(E=e_Y-e_X\sim\N(\mu_{e_Y}-\mu_{e_X},\,\sigma_{e_X}^2+\sigma_{e_Y}^2)\),
\begin{equation}
    A=E+d_X-d_Y,\qquad B=E,\qquad G=E-d_Y,\qquad H=-E-d_X.
\end{equation}
Symmetrically, the latent end difference is the finish/finish primitive \(B\). The midpoint form of the body sits between the two: with \(Z=t_Y-t_X\), \(A=Z+\tfrac12(d_X-d_Y)\), \(B=Z-\tfrac12(d_X-d_Y)\), \(G=Z-\tfrac12(d_X+d_Y)\), and \(H=-Z-\tfrac12(d_X+d_Y)\).

\paragraph{Start and end.} Take \(X=[s_X,\,e_X]\) with independent Gaussian boundaries, so \(a_X=s_X\), \(b_X=e_X\), and the duration is derived, \(d_X=e_X-s_X\). The four differences are then the boundary differences read literally,
\begin{equation}
    A=s_Y-s_X,\qquad B=e_Y-e_X,\qquad G=s_Y-e_X,\qquad H=s_X-e_Y,
\end{equation}
and the truncation \(d_X\ge0\) is the half-space \(e_X\ge s_X\): the law is that of independent boundaries conditioned on each interval being well formed, with normalizer \(Q_X=\PhiCDF\big((\mu_{e_X}-\mu_{s_X})/\sqrt{\sigma_{s_X}^2+\sigma_{e_X}^2}\big)\). In the body's coordinates this is a midpoint--duration pair that is \emph{correlated}: \(t_X=\tfrac12(s_X+e_X)\) and \(D_X=e_X-s_X\) are jointly Gaussian with
\begin{equation}
    \sigma_{t_X}^2=\tfrac14\big(\sigma_{s_X}^2+\sigma_{e_X}^2\big),\qquad
    \sigma_{D_X}^2=\sigma_{s_X}^2+\sigma_{e_X}^2,\qquad
    \operatorname{Cov}(t_X,D_X)=\tfrac12\big(\sigma_{e_X}^2-\sigma_{s_X}^2\big),
    \label{eq:endpoint-moments}
\end{equation}
so a sharp start (\(\sigma_{s_X}=0\)) is the start form, a sharp end is the end form, and \emph{equally} uncertain boundaries give \(\operatorname{Cov}(t_X,D_X)=0\) with \(\sigma_{D_X}=2\sigma_{t_X}\): the independent midpoint--duration model of the body is exactly the symmetric-boundary case. Conversely, independent latent \((t,D)\) with arbitrary \(\sigma_t,\sigma_D\) are boundaries \(s=t-D/2\), \(e=t+D/2\) of equal variance \(\sigma_t^2+\sigma_D^2/4\) and correlation \((\sigma_t^2-\sigma_D^2/4)/(\sigma_t^2+\sigma_D^2/4)\). The independence assumed in Section~\ref{sec:representation} is therefore not a different model from a boundary one, but the choice of which pair is taken uncorrelated.

Each form is the same inequality system in different coordinates: the latent vector \((\,\cdot\,,D_X,D_Y)\) and the truncation \(D_X,D_Y\ge0\) are unchanged---the boundary form either in its own four coordinates or, through Equation~\eqref{eq:endpoint-moments}, in \((Z,D_X,D_Y)\) with a non-diagonal \(\Sigma\)---and only the linear map to \((A,B,G,H)\) differs, so the orthant machinery is identical. What the choice \emph{does} change is the modelled law---which two quantities are taken as independent and Gaussian---so one selects the pair best matched to the data (start times for events with a known onset, end times for deadlines, the midpoint for symmetric jitter, the two boundaries when each is timestamped on its own), while the derivation, closed forms, and partition carry over verbatim.

\subsection{General Linear Form}

The start, end, and midpoint forms are instances of one template, and nothing in the derivation, the closed forms, or the partition depends on the choice. Collect every latent Gaussian temporal quantity of the two objects into a vector \(\bm U\sim\N(\bm\mu,\Sigma)\)---in the midpoint form of Section~\ref{sec:interval-interval}, \(\bm U=(Z,D_X,D_Y)\transpose\); in the start and end forms, \((S,D_X,D_Y)\transpose\) and \((E,D_X,D_Y)\transpose\); in the boundary form, \((S_X,E_X,S_Y,E_Y)\transpose\). The truncation \(\bm D\ge\bm0\) of Section~\ref{sec:representation} is a set of half-spaces of \(\bm U\)---its duration coordinates where the form has them, the rows \(E_X-S_X\ge0\) where it does not---with normalizer \(Q=\prod_X\Prob(D_X\ge0)\). The form enters through one object only---a linear difference map---while the relations are read the same way for every form.

First, the boundaries are linear in \(\bm U\), so the four boundary differences of Table~\ref{tab:prob-allen-ineq} are too:
\begin{equation}
    \begin{pmatrix}A\\B\\G\\H\end{pmatrix}=M\,\bm U ,
    \label{eq:diff-map}
\end{equation}
with the difference map \(M\) read off the boundaries---the only object the parameterization changes. From the start and end differences above and the midpoint form of Section~\ref{sec:interval-interval},
\begin{equation}
    M_{\mathrm M}=\begin{pmatrix}1&\tfrac12&-\tfrac12\\ 1&-\tfrac12&\tfrac12\\ 1&-\tfrac12&-\tfrac12\\ -1&-\tfrac12&-\tfrac12\end{pmatrix},\quad
    M_{\mathrm S}=\begin{pmatrix}1&0&0\\ 1&-1&1\\ 1&-1&0\\ -1&0&-1\end{pmatrix},\quad
    M_{\mathrm E}=\begin{pmatrix}1&1&-1\\ 1&0&0\\ 1&0&-1\\ -1&-1&0\end{pmatrix},
    \label{eq:Mforms}
\end{equation}
each acting on its form's \(\bm U\) (columns ordered as that form's coordinates); \(M_{\mathrm M}\) is the matrix \(M\) of Appendix~\ref{app:factorization}. The boundary form has the simplest map of all, acting on \((S_X,E_X,S_Y,E_Y)\transpose\),
\begin{equation}
    M_{\mathrm{SE}}=\begin{pmatrix}-1&0&1&0\\ 0&-1&0&1\\ 0&-1&1&0\\ 1&0&0&-1\end{pmatrix},
    \label{eq:Mse}
\end{equation}
each row one of the four boundary differences written out. No constant term is needed in any form---the differences are purely linear in \(\bm U\).

Second, each relation \(R\) is a sign pattern on those four differences---its signature \(\sigma_R\in\{-1,0,+1\}^4\) of Table~\ref{tab:primitive-signs} (the \(+,0,-\) of that table), the same for every form. The signature \emph{is} the selector: a matrix \(S_R\) with entries in \(\{-1,0,+1\}\) acting on \((A,B,G,H)\transpose\), where a strict sign \(\sigma_v\in\{\pm1\}\) gives the single row \(-\sigma_v\,v\le-\tau\) and a coincidence \(\sigma_v=0\) the contact pair \(+v\le\tau,\ -v\le\tau\). For the six strict relations this is just \(S_R=-\operatorname{diag}(\sigma_R)\) with \(\bm b_R=-\tau\mathbf 1\); each coincidence of a contact relation replaces that diagonal entry by a \(\pm\) row pair. Explicitly, \before\ (\(\sigma=(+,+,+,-)\), strict) and \meets\ (\(\sigma=(+,+,0,-)\), with \(G\) in contact) give
\begin{equation}
    S_{\before}=\begin{pmatrix}-1&0&0&0\\ 0&-1&0&0\\ 0&0&-1&0\\ 0&0&0&1\end{pmatrix},\qquad
    S_{\meets}=\begin{pmatrix}-1&0&0&0\\ 0&-1&0&0\\ 0&0&1&0\\ 0&0&-1&0\\ 0&0&0&1\end{pmatrix},
    \label{eq:SR-examples}
\end{equation}
with \(\bm b_{\before}=-\tau\mathbf 1\) and \(\bm b_{\meets}=\tau(-1,-1,1,1,-1)\transpose\); the coincident \(G\) of \meets\ supplies the two band rows \(\pm G\le\tau\), and rows beyond the minimal defining set of Table~\ref{tab:prob-allen-ineq} are entailed and harmless. Substituting the difference map~\eqref{eq:diff-map} carries the constraints to \(\bm U\),
\begin{equation}
    S_R\,M\,\bm U\le\bm b_R ,
    \label{eq:LR-compose}
\end{equation}
which is the relation system \(L_R(\bm U)=\bm A_R\bm U-\bm b_R\le\bm0\) of Equation~\eqref{eq:LR-def} with \(\bm A_R=S_R M\); the truncation \(\bm D\ge\bm0\) enters separately, as the conditioning event below. The parameterization thus lives entirely in \(M\) and the relation entirely in its signature \(\sigma_R\) (Table~\ref{tab:primitive-signs})---a relation is a choice of which signs the four differences take, not an additive shift of them. Hence
\begin{equation}
    P(R)=\frac{\Prob\!\big(L_R(\bm U)\le\bm0,\ \bm D\ge\bm0\big)}{Q},
\end{equation}
the conditional Gaussian mass of a polytope, evaluated by standardizing its affine rows with \(\bm\mu\) and \(\Sigma\) as in Section~\ref{sec:interval-interval}. Between parameterizations only \((M,\bm\mu,\Sigma)\) change; the signatures of Table~\ref{tab:primitive-signs}, the closed forms, and the partition carry over verbatim.

\paragraph{Example: \before\ in two coordinate systems.} From Table~\ref{tab:prob-allen-ineq}, \before\ is the single row \(G>\tau\). In the midpoint form \(G=Z-\tfrac12(D_X+D_Y)\), so
\begin{equation}
    P(\before)=\frac{\Prob\!\big(Z-\tfrac12(D_X+D_Y)>\tau,\ D_X,D_Y\ge0\big)}{Q};
\end{equation}
in the start form \(G=S-D_X\), giving \(\Prob(S-D_X>\tau,\ D_X,D_Y\ge0)/Q\). These are one probability written in different latent coordinates, and both collapse, when the durations are sharp, to the one-dimensional tail \(\tfrac12\erfcop(\cdot)\) of Section~\ref{sec:point-point}.

\paragraph{Example: correlated quantities.} If the two events share an onset cue, their start times are correlated and \(\Sigma\) is no longer diagonal. Nothing else changes: \(L_R\) and the truncation are as above, and the off-diagonal entries of \(\Sigma\) merely propagate into the standardizing covariance of the orthant integral. The independence assumed in Section~\ref{sec:representation} is thus a convenience, not a requirement.

\section{Primitive Factorization of the Probabilistic Algebra}
\label{app:factorization}

This appendix derives Proposition~\ref{prop:factorization}: the probabilistic algebra of Section~\ref{sec:interval-interval} is exactly the joint law of the four temporal primitives, and its primitive marginals have one-dimensional Gaussian closed forms.

\paragraph{Latent model and the primitive map.}
Let \(\bm U=(Z,D_X,D_Y)\transpose\), with \(Z=t_Y-t_X\sim\N(\Delta,\sigma_Z^2)\), \(\Delta=\mu_Y-\mu_X\), \(\sigma_Z^2=\sigma_X^2+\sigma_Y^2\), independent of the latent durations \(D_X,D_Y\) conditioned on \(D_X,D_Y\ge0\). The four boundary differences are linear in \(\bm U\):
\begin{equation}
\begin{pmatrix}A\\B\\G\\H\end{pmatrix}=M\bm U,\qquad
M=\begin{pmatrix}1&\tfrac12&-\tfrac12\\[2pt] 1&-\tfrac12&\tfrac12\\[2pt] 1&-\tfrac12&-\tfrac12\\[2pt] -1&-\tfrac12&-\tfrac12\end{pmatrix}.
\end{equation}
The rows obey \(M_H=M_G-M_A-M_B\), so \(H=G-A-B\) and \(\operatorname{rank}M=3\); the triple \((A,B,G)\) is an invertible image of \(\bm U\). Write \(\Phi(\bm U)=(\operatorname{sgn}_\tau A,\operatorname{sgn}_\tau B,\operatorname{sgn}_\tau G,\operatorname{sgn}_\tau H)\) and let \(M_v\) denote the row of \(M\) for difference \(v\).

\paragraph{Step 1: each relation is a primitive fiber.}
Table~\ref{tab:prob-allen-ineq} defines each relation by sign constraints on \((A,B)\) and, on the two diagonal cells, on \(G\) or \(H\); the remaining signs are entailed on the feasible cone. For instance \(G>\tau\) means \(b_X<a_Y-\tau\); with \(a_X<b_X\) and \(a_Y<b_Y\) this forces \(a_X<b_X<a_Y<b_Y\), hence \(A>\tau\), \(B>\tau\), \(H<-\tau\), so \(\{G>\tau\}=\Phi^{-1}(+,+,+,-)=\{\before\}\). The same entailment closes every row, giving \(\{R\}=\Phi^{-1}(\sigma_R)\); the thirteen signatures \(\sigma_R\) of Table~\ref{tab:primitive-signs} are exactly the feasible values of \(\Phi\). This is the partition of Proposition~\ref{prop:partition} read in primitive coordinates.

\paragraph{Step 2: relation probability is a joint primitive integral.}
Because \(\{R\}=\bigcap_v\{\operatorname{sgn}_\tau v=\sigma_R(v)\}\),
\begin{equation}
    P(R)=\frac{1}{Q_XQ_Y}\int_{\mathbb R^3}\Big[\prod_{v}\mathbf 1\{\operatorname{sgn}_\tau v=\sigma_R(v)\}\Big]\,
    \mathbf 1\{D_X\ge0,\,D_Y\ge0\}\,\phi(\bm u;\bm\mu,\Sigma)\,d\bm u,
\end{equation}
with \(Q_X=\Prob(D_X\ge0)\) and \(Q_Y=\Prob(D_Y\ge0)\). Each factor is a slab or half-space in \(\bm U\): \(\{\operatorname{sgn}_\tau v=+\}=\{M_v\bm U>\tau\}\), \(\{\operatorname{sgn}_\tau v=-\}=\{M_v\bm U<-\tau\}\), and the contact band \(\{\operatorname{sgn}_\tau v=0\}=\{|M_v\bm U|\le\tau\}\). Hence \(\{R\}\) is a (possibly lower-dimensional) polytope and \(P(R)\) is the Gaussian measure of that polytope --- the conditional multivariate-Gaussian orthant probability of Section~\ref{sec:interval-interval}. Each contact band expands by inclusion--exclusion,
\begin{equation}
    \mathbf 1\{|M_v\bm U|\le\tau\}=\mathbf 1\{M_v\bm U\le\tau\}-\mathbf 1\{M_v\bm U<-\tau\},
\end{equation}
so \(P(R)\) is a signed sum of pure orthant probabilities; the number of bands equals the contact count \(c(R)\) of Proposition~\ref{prop:contact}.

\paragraph{Step 3: primitive marginals in closed form.}
Marginalizing the joint over three coordinates gives the coarse predicate \(P(\operatorname{sgn}_\tau v=s)=\sum_{R:\,\sigma_R(v)=s}P(R)\). Each is one-dimensional: conditioning on \((d_X,d_Y)\) leaves \(Z\sim\N(\Delta,\sigma_Z^2)\) as the only free Gaussian, and each difference is \(\pm Z\) plus a duration shift. With the expectation \(\E_d\) taken over the truncated durations, the gap \(G=Z-\tfrac12(d_X+d_Y)\) gives
\begin{align}
    P(\operatorname{sgn}_\tau G={+})&=\E_{d}\!\left[\PhiCDF\!\left(\frac{\Delta-\tau-\tfrac12(d_X+d_Y)}{\sigma_Z}\right)\right]=P(\before),\\
    P(\operatorname{sgn}_\tau G={0})&=\E_{d}\!\left[\PhiCDF\!\left(\frac{\Delta+\tau-\tfrac12(d_X+d_Y)}{\sigma_Z}\right)-\PhiCDF\!\left(\frac{\Delta-\tau-\tfrac12(d_X+d_Y)}{\sigma_Z}\right)\right]\notag\\
    &=P(\meets),
\end{align}
and the start/start difference \(A=Z+\tfrac12(d_X-d_Y)\) yields shared start,
\begin{equation}
\begin{split}
    P(\operatorname{sgn}_\tau A={0})&=\E_{d}\!\left[\PhiCDF\!\left(\frac{\tau+\tfrac12(d_X-d_Y)-\Delta}{\sigma_Z}\right)-\PhiCDF\!\left(\frac{-\tau+\tfrac12(d_X-d_Y)-\Delta}{\sigma_Z}\right)\right]\\
    &=P(\starts)+P(\equals)+P(\startedby).
\end{split}
\end{equation}
These one-dimensional forms agree with the full multivariate engine to Monte-Carlo precision.

\paragraph{Step 4: the joint does not factor.}
Since \(\operatorname{rank}M=3<4\), the primitives obey the deterministic identity \(H=G-A-B\), and \(\Sigma\) couples them; thus \(P(R)\ne\prod_v P(\operatorname{sgn}_\tau v=\sigma_R(v))\) except in degenerate limits. Two regimes make the gap explicit. \emph{Entailment:} \(G>\tau\Rightarrow A,B>\tau\), so \(P(\before)=P(\operatorname{sgn}_\tau G={+})\), well below the product of the three marginals. \emph{Conflict:} at a shared midpoint (\(\Delta=0\), small \(\sigma_Z\)) one has \(A\approx-B\), so \(\{A>\tau\}\) and \(\{B>\tau\}\) are nearly mutually exclusive and their joint falls far below the product. The covariance \(\Sigma\) is precisely the dependence a logical primitive ontology such as CIDOC~CRM omits and that the probabilistic algebra restores.

\section{Software Interface and Implementation Notes}
\label{app:software}

The algebra ships as an open, MIT-licensed, tested Python package (import name \texttt{paa}, distribution \texttt{probabilistic-\allowbreak allen-\allowbreak algebra}). This appendix documents the user-facing interface---the functions you need to call to obtain the probabilities derived in the body---and the implementation choices that bear on the paper's claims, in particular cost and numerical accuracy. Every quantity below is the analytic object defined earlier, not a Monte-Carlo estimate; sampling is used only for validation (Section~\ref{sec:partition}).

\paragraph{Availability and contents.}
The package is at \url{https://github.com/HRI-EU/probabilistic-allen-algebra} and installs with \texttt{pip install probabilistic-allen-algebra} (Python~\(\ge3.11\), depending only on \texttt{numpy} and \texttt{scipy}). It is analytic and self-contained: it needs no data, no network, and no API access, and reproduces every number and figure in this paper from source. Six modules realize the constructions of the preceding sections---\texttt{intervals} (the uncertain objects of Section~\ref{sec:representation}), \texttt{relations} (the thirteen orthant probabilities of Section~\ref{sec:interval-interval}), \texttt{taxonomy} (the partition tree and coarse predicates of Section~\ref{sec:taxonomy}), \texttt{decode} (the best-relation selectors of Section~\ref{sec:best-relation}), \texttt{primitives} (the boundary-difference factorization of Appendix~\ref{app:factorization}), and \texttt{plotting} (the figure generators)---and a command-line interface (\texttt{paa prob\,|\,coarse\,|\,decode\,|\,\ldots}, equivalently \texttt{python -m paa}) exposes the same queries from the shell. A \texttt{pytest} suite enforces the partition, the Monte-Carlo agreement, the limiting reductions, and the invariances (Section~\ref{sec:validation}), so every reported result is reproducible from source. The optional \texttt{viz} extra adds \texttt{matplotlib} for the figure generators; nothing else is required.

\paragraph{Choosing the right call.} Three decisions route a problem to its formula and function (Figure~\ref{fig:decision}, Table~\ref{tab:decision}): the object types, the quantities known for each interval, and the desired output.

\begin{figure}[t]
\centering
\begin{tikzpicture}[
  node distance=4mm,
  stage/.style={draw, rounded corners, fill=black!4, align=left,
                text width=0.92\linewidth, inner sep=5pt, font=\small}]
  \node[stage] (s1) {\textbf{1.\ What are the two objects \(X,Y\)?}\quad both points \(\rightarrow\) point--point (3 relations, \(\mathrm{erf}/\mathrm{erfc}\)); one point, one interval \(\rightarrow\) point--interval (5 relations, \(\Phi_2\)); two intervals \(\rightarrow\) interval--interval (13 relations, \(\Phi_3/\Phi_4\)).};
  \node[stage, below=of s1] (s2) {\textbf{2.\ For each interval, which quantities are known?}\quad position \(+\) duration \(\rightarrow\) \texttt{IntervalGaussian(mu\_t, sigma\_t, mu\_d, sigma\_d)}; start \(+\) duration \(\rightarrow\) \texttt{from\_start(\ldots)}; end \(+\) duration \(\rightarrow\) \texttt{from\_end(\ldots)}; a point \(\rightarrow\) \texttt{point(mu\_t, sigma\_t)}.};
  \node[stage, below=of s2] (s3) {\textbf{3.\ Choose the tolerance \(\tau\); then the query.}\quad all thirteen probabilities \(\rightarrow\) \texttt{relation\_probabilities(e, r, tau)}; a coarse family \(\rightarrow\) \texttt{coarse\_predicates}\,/\,\texttt{refine}; a single best label \(\rightarrow\) \texttt{hierarchical\_decode}\,/\,\texttt{map\_relation}.};
  \draw[->] (s1) -- (s2);
  \draw[->] (s2) -- (s3);
\end{tikzpicture}
\caption{Using the formulas and software in three decisions. The object types fix the closed-form family and how many relations are non-trivial (Sections~\ref{sec:point-point}--\ref{sec:interval-interval}); the known quantities fix each interval's constructor (Section~\ref{sec:parameterization}); the desired output fixes the call. The first decision is conceptual---points are degenerate intervals, so \texttt{relation\_probabilities} serves all three cases through one call---so the chart maps the formulas rather than forcing a branch in code.}
\Description{A three-step decision diagram: choose the object types (point or interval), choose which quantities are known, and choose the tolerance; each step names the closed form and the package function to call.}
\label{fig:decision}
\end{figure}

\begin{table}[t]
\centering\small
\begin{tabular}{lccl}
\toprule
objects \((X,Y)\) & relations & closed form & constructor(s) \\
\midrule
point, point       & 3  & \(\mathrm{erf}/\mathrm{erfc}\) & \texttt{point} \\
point, interval    & 5  & bivariate \(\Phi_2\)          & \texttt{point} \(+\) interval ctor \\
interval, interval & 13 & \(\Phi_3/\Phi_4\) orthant     & interval ctor \\
\bottomrule
\end{tabular}
\caption{Object types fix the closed-form family and the number of non-trivial relations; Step~2 of Figure~\ref{fig:decision} fixes the constructor for each interval.}
\label{tab:decision}
\end{table}

\paragraph{Constructing uncertain objects.}
An uncertain interval is created as \texttt{IntervalGaussian(\allowbreak mu\_t, sigma\_t, mu\_d, sigma\_d)}, a Gaussian midpoint \(t\sim\N(\mu_t,\sigma_t^2)\) and a latent duration \(D\sim\N(\mu_d,\sigma_d^2)\) truncated to \(D\ge0\) (Section~\ref{sec:representation}). A time point is the helper \texttt{point(mu\_t, sigma\_t=0.0)}, the degenerate interval with zero duration. Deterministic limits are admissible: \(\sigma_t=0\) fixes the midpoint and \(\sigma_d=0\) fixes the duration.

\paragraph{Computing relation probabilities.}
The single entry point is \texttt{relation\_probabilities(e, r, tau=0.0)}, which returns the thirteen probabilities for the event \texttt{e} against the reference \texttt{r} as a dictionary keyed by relation name. The object-oriented form \texttt{ProbabilisticAllenRelations(e, r)} exposes one method per relation (e.g.\ \texttt{.before(tau)}) and \texttt{.all\_relations(tau)}; the tolerance \(\tau\) is the single partition parameter of Section~\ref{sec:partition}. Listing~\ref{lst:quickstart} runs the full flow end to end.

\lstset{language=Python, basicstyle=\footnotesize\ttfamily, keywordstyle=\color{blue!55!black},
  commentstyle=\color{gray!75!black}, stringstyle=\color{green!40!black},
  showstringspaces=false, columns=fullflexible, frame=single, framesep=4pt,
  xleftmargin=4pt, breaklines=true}
\begin{lstlisting}[caption={Minimal end-to-end example: construct two uncertain intervals, compute the thirteen relation probabilities (a partition summing to one), and traverse the taxonomy of Section~\ref{sec:taxonomy}---coarse-predicate node masses, a within-parent refinement, and the hierarchical decode. Each \texttt{print} displays the value shown in the trailing comment.}, label={lst:quickstart}]
from paa import (IntervalGaussian, relation_probabilities,
                 coarse_predicates, refine, hierarchical_decode)

e = IntervalGaussian(mu_t=0.0, sigma_t=1.0, mu_d=2.0, sigma_d=0.3)
r = IntervalGaussian(mu_t=3.0, sigma_t=1.0, mu_d=1.5, sigma_d=0.4)

p = relation_probabilities(e, r, tau=0.25)      # dict: 13 relations -> probability
print(sum(p.values()))                          # ~ 1.0   (a true partition)

c = coarse_predicates(p)                         # masses of every taxonomy node
print(c["separated"], c["non_separated"])       # parent masses, sum to 1
print(refine(p, "precede"))                      # {'before':.., 'meets':..}, sums to 1
print(hierarchical_decode(p))                    # [('separated',..),('precede',..),('before',..)]
\end{lstlisting}

\paragraph{Coarse predicates, refinement, and selection.}
Given the leaf dictionary \texttt{p}, the taxonomy calculus of Section~\ref{sec:taxonomy} is realized by \texttt{coarse\_predicates(p)} (the probability of every node, each the sum of its leaves) and \texttt{refine(p, parent)} (the conditional distribution \(P(\text{leaf}\mid\text{parent})\)). The partition tree itself is the data structure \texttt{TREE}, with \texttt{LEAVES}, \texttt{leaves\_of}, and the overlapping non-strict containment unions \texttt{VIEWS}. Three best-fitting-relation selectors implement the discussion of Section~\ref{sec:best-relation}: \texttt{most\_probable\_relation} (flat \(\argmax\), measure-biased), \texttt{hierarchical\_decode} (top-down maximum-a-posteriori over the taxonomy), and \texttt{map\_relation} (the relation of the modal arrangement). The primitive decomposition of Appendix~\ref{app:factorization} is exposed by \texttt{primitive\_probabilities} (the four three-way marginals), \texttt{decompose}, \texttt{contacts}, and the signatures \texttt{CANONICAL\_SIGNS}.

\paragraph{Implementation and cost.}
Inference is analytic and sampling-free. Each relation is the Gaussian measure of a polytope in the latent \(\bm U=(Z,D_X,D_Y)\transpose\) with \(Z=t_Y-t_X\) (Section~\ref{sec:interval-interval}), evaluated as a multivariate-normal CDF in \emph{at most three} dimensions; \before\ and \after\ collapse to a single half-space (a one-dimensional \(\PhiCDF\) over the gap), and the point--point and point--interval cases reduce to the error function and a bivariate \(\Phi_2\) (Appendix~\ref{app:limits}). A call to \texttt{relation\_probabilities} therefore returns all thirteen from a handful of low-dimensional CDF evaluations in time independent of the temporal scale and of any sample count. Degenerate inputs are handled exactly rather than as a numerical limit: any latent coordinate with standard deviation at or below \texttt{eps} (default \(10^{-12}\)) is removed from the Gaussian system and folded into a deterministic constant, so points and fixed durations lower the CDF dimension and realize the collapse limits of Appendix~\ref{app:limits} without singular covariances; the lower-dimensional contact bands are evaluated with a singular-aware CDF. Every probability is divided by the truncation normalizer \(Q=\Prob(D_X\ge0)\,\Prob(D_Y\ge0)\) that conditions on non-negative durations. Because no random number generator is involved at inference, results are deterministic and reproducible.

\paragraph{Numerical validation.}
The analytic probabilities agree with an independent Monte-Carlo classifier to a maximum absolute error of \(\approx1\text{--}2\times10^{-4}\) at \(4\times10^6\) samples, and the thirteen sum to one within that noise for every tolerance tested, confirming Proposition~\ref{prop:partition} (Section~\ref{sec:validation}); scale invariance and converse symmetry hold to the same tolerance. These checks, together with the limiting reductions and the figure generators, are shipped as the package test suite so that all reported results are reproducible from source.

\bibliographystyle{plain}
\bibliography{references}

\end{document}